\documentclass{article}
\pdfoutput=1

\usepackage[final,nonatbib]{neurips_2023}

\usepackage[utf8]{inputenc} 
\usepackage[T1]{fontenc}    
\usepackage{url}            
\usepackage{hyperref} 
\usepackage{booktabs}       
\usepackage{amsfonts}       
\usepackage{nicefrac}       
\usepackage{microtype}      
\usepackage{xcolor}         
\usepackage{fancyhdr}
\usepackage{wrapfig}

\usepackage{bbm}
\usepackage{enumerate}
\usepackage{amsmath}
\usepackage{amsthm}
\usepackage{makecell}
\usepackage{amssymb}
\usepackage{amsfonts}
\usepackage{mathrsfs}
\usepackage{mathtools}
\usepackage[all]{xy}

\usepackage{graphicx}
\usepackage{caption}

\usepackage{nicefrac}

\newcommand{\pa}[1]{ \left({#1}\right) }
\newcommand{\ha}[1]{ \left[{#1}\right] }
\newcommand{\ca}[1]{ \left\{{#1}\right\} }
\newcommand{\inner}[1]{\left\langle #1 \right\rangle}

\newcommand{\norm}[1]{\left\lVert #1 \right\rVert}

\newcommand{\R}{\mathbb{R}}

\newcommand{\mcH}{\mathcal{H}}

\newcommand{\mcK}{\mathcal{K}}

\renewcommand{\d}[1]{\mathop{\mathrm{d} #1 }}
\DeclarePairedDelimiterX{\infdivx}[2]{(}{)}{ #1\;\delimsize\|\;#2 }

\makeatletter
\newcommand{\distas}[1]{\mathbin{\overset{#1}{\kern\z@\sim}}}%
\makeatother

\DeclareMathOperator\mathExp{\mathbb{E}}

\newcommand{\E}{\mathExp}

\DeclareMathOperator{\diag}{diag}
\DeclareMathOperator{\tr}{tr}
\DeclareMathOperator{\rank}{rank}
\DeclareMathOperator*{\argmin}{argmin}

\renewcommand{\d}[1]{\mathop{\mathrm{d} #1 }}

\makeatletter
\DeclareFontFamily{U}  {MnSymbolF}{}
\DeclareSymbolFont{symbolsMN}{U}{MnSymbolF}{m}{n}
\SetSymbolFont{symbolsMN}{bold}{U}{MnSymbolF}{b}{n}
\DeclareFontShape{U}{MnSymbolF}{m}{n}{
    <-6>  MnSymbolF5
   <6-7>  MnSymbolF6
   <7-8>  MnSymbolF7
   <8-9>  MnSymbolF8
   <9-10> MnSymbolF9
  <10-12> MnSymbolF10
  <12->   MnSymbolF12}{}
\DeclareFontShape{U}{MnSymbolF}{b}{n}{
    <-6>  MnSymbolF-Bold5
   <6-7>  MnSymbolF-Bold6
   <7-8>  MnSymbolF-Bold7
   <8-9>  MnSymbolF-Bold8
   <9-10> MnSymbolF-Bold9
  <10-12> MnSymbolF-Bold10
  <12->   MnSymbolF-Bold12}{}
\DeclareMathSymbol{\tbigtimes}{\mathop}{symbolsMN}{2}
\newcommand*{\bigtimes}{%
  \DOTSB
  \tbigtimes
  \slimits@ 
}
\makeatother

\newcommand\restr[2]{{
  \left.\kern-\nulldelimiterspace 
  #1 
  \vphantom{\big|} 
  \right|_{#2} 
  }}
\usepackage{amsmath,amsfonts,amssymb}
\usepackage{mathtools}
\usepackage{amsthm} 
\usepackage{latexsym}
\usepackage{relsize}

\usepackage[capitalise]{cleveref}
\usepackage{xcolor}
\usepackage{dsfont}

\newcommand{\defeq}{\stackrel{\text{def}}{=}}

\newcommand{\A}{\mathcal{A}}

\newcommand{\K}{\ensuremath{\mathcal K}}

\def\regret{\mbox{{Regret}}}

\newcommand{\ignore}[1]{}

\newcommand{\xc}[1]{\noindent{\textcolor{blue}{\{{\bf XC:} {\em #1}\}}}}
\newcommand{\vlad}[1]{\noindent{\textcolor{purple}{\{{\bf VF:} {\em #1} \}}}}

\usepackage{wasysym}

\theoremstyle{plain}
\newtheorem{theorem}{Theorem}
\newtheorem{lemma}[theorem]{Lemma}
\newtheorem{corollary}[theorem]{Corollary}

\newtheorem{observation}[theorem]{Observation}
\newtheorem{assumption}{Assumption}

\newtheorem*{theorem*}{Theorem}
\newtheorem*{lemma*}{Lemma}
\newtheorem*{corollary*}{Corollary}
\newtheorem*{proposition*}{Proposition}
\newtheorem*{claim*}{Claim}
\newtheorem*{fact*}{Fact}
\newtheorem*{observation*}{Observation}
\newtheorem*{assumption*}{Assumption}

\theoremstyle{definition}
\newtheorem{definition}[theorem]{Definition}
\newtheorem{remark}[theorem]{Remark}

\newtheorem*{definition*}{Definition}
\newtheorem*{remark*}{Remark}
\newtheorem*{example*}{Example}

 \theoremstyle{plain}
\newtheorem*{theoremaux}{\theoremauxref}
\gdef\theoremauxref{1}

\DeclareMathAlphabet{\mathbfsf}{\encodingdefault}{\sfdefault}{bx}{n}

\let\Pr\relax
\DeclareMathOperator{\Pr}{\mathbb{P}}

\newcommand{\trace}{\mathrm{tr}}

\newcommand{\reals}{\mathbb{R}}
\newcommand{\eps}{\varepsilon}
\newcommand{\ep}{\varepsilon}

\renewcommand{\leq}{~\le~}

\let\oldtfrac\tfrac
\renewcommand{\tfrac}[2]{\smash{\oldtfrac{#1}{#2}}}

\let\nablaold\nabla
\renewcommand{\nabla}{\nablaold\mkern-2.5mu}

\usepackage{amsthm}
\usepackage{etoolbox}
\usepackage{algorithm, amsmath}
\usepackage{tablefootnote}
\usepackage{graphicx}
\usepackage{placeins}
\usepackage{subfigure}
\usepackage{hyperref} 
\usepackage[numbers,square]{natbib}

\let\classOUTPUT\OUTPUT
\let\OUTPUT\relax
\usepackage{algorithmic}

\let\algoOUTPUT\OUTPUT

\let\OUTPUT\classOUTPUT
\AtBeginEnvironment{algorithmic}{}
\newcommand{\bvec}[1]{ \overline{\text{vec}}\left({#1}\right) }

\title{Sketchy: Memory-efficient Adaptive Regularization with Frequent Directions}

\author{
  Vladimir Feinberg\\
  Google DeepMind\\
\href{mailto:vladf@google.com}{vladf@google.com}
  \And
   Xinyi Chen\\
  Princeton University \\
  Google DeepMind
  \And
   Y. Jennifer Sun\\
  Princeton University\\
  Google DeepMind
  \And
  Rohan Anil\\
  Google DeepMind
  \And
  Elad Hazan \\
  Princeton University \\
  Google DeepMind
}

\begin{document}

\maketitle

\begin{abstract}
 Adaptive regularization methods that exploit more than the diagonal entries exhibit state of the art performance for many tasks, but can be prohibitive in terms of memory and running time. We find the spectra of the Kronecker-factored gradient covariance matrix in deep learning (DL) training tasks are concentrated on a small leading eigenspace that changes throughout training, motivating a low-rank sketching approach. We describe a generic method for reducing memory and compute requirements of maintaining a matrix preconditioner using the Frequent Directions (FD) sketch. While previous approaches have explored applying FD for second-order optimization, we present a novel analysis which allows efficient interpolation between resource requirements and the degradation in regret guarantees with rank $k$: in the online convex optimization (OCO) setting over dimension $d$, we match full-matrix $d^2$ memory regret using only $dk$ memory up to additive error in the bottom $d-k$ eigenvalues of the gradient covariance. Further, we show extensions of our work to Shampoo, resulting in a method competitive in quality with Shampoo and Adam, yet requiring only sub-linear memory for tracking second moments.
\end{abstract}

\section{Introduction}

DL optimization commonly relies on adaptive gradient methods, namely the Adam optimizer \cite{kingma2015adam}. It differs from stochastic gradient descent in that the learning rate is a structured diagonal matrix built from previous gradients rather than a scalar. In full matrix AdaGrad \citep{duchi2011adaptive}, the inverse matrix square root of the sum of outer products of previous gradients is the learning rate.

Full matrix preconditioning is impractical for modern deep learning architectures: for instance, the ResNet-50 architecture \cite{he2016deep} has over 23 million parameters, requiring more than 2 petabytes to represent its gradient covariance. Thus, diagonal preconditioning methods remain popular. However, previous work has demonstrated state-of-the-art results in some settings, such as large-batch data parallel training, for nondiagonal forms of preconditioning \cite{martens2015optimizing,gupta2018shampoo,ggt,chen2019extreme,anil2019memory, anil2020scalable}. In particular, Shampoo \citep{gupta2018shampoo, anil2020scalable} introduces a factorization of full matrix preconditioning method with adoption in large-scale industrial applications such as training Google's ads click-through-rate model \cite{anil2022factory}. Furthermore, as hardware evolves, memory efficiency becomes an increasing concern, as ``logic improves much faster than wires and SRAM, so logic is relatively free'' \citep{jouppi2021ten}: from TPUv2 to TPUv3, per-chip \texttt{bfloat16} operations per second improved $2.67\times$ but memory bandwidth only improved $1.29\times$. GPUs exhibit a similar pattern for compute and memory increase, at $5\times$ and $2.2\times$, for V100 to A100 \citep{dally2021evolution}.

Investigation into the Kronecker-factored gradient covariance matrix reveals a concentrated, but changing, spectrum (Fig.~\ref{fig:intrinsic}), suggesting the majority of the spectral mass can be represented by a low-rank matrix, albeit rotating over time. The Frequent Directions (FD) sketch provides a mechanism to track the top eigenvectors without materializing the full covariance matrix, as proposed in  \cite{ghashami2016frequent}. Is a large portion of the spectral mass sufficient to retain the performance of adaptive regularization in theory and practice? In this work, we investigate this hypothesis.
\begin{itemize}
    \item In the setting of online convex optimization, by applying a dynamic diagonal regularization to the FD sketch, we can \textbf{recover full-matrix AdaGrad regret up to additive spectral terms under a memory constraint}, providing a novel guarantee without curvature assumptions (Sec.~\ref{sec:fd-AdaGrad}). Rigorously composing our approach with Shampoo (Sec.~\ref{sec:fd-shampoo}) unlocks a \textbf{second-order algorithm which requires sub-linear memory} for its accumulators.
    \item By modifying FD for exponential moving averages (Sec.~\ref{sec:exp-weighted-fd}), we demonstrate a practical algorithm competitive with at-least-linear memory Shampoo and Adam in three modern DL settings (Sec.~\ref{sec:generalization-benchmarks}). While previous work \citep{anil2019memory, shazeer2018adafactor} shows rank-1 preconditioners are effective for trading off quality for memory, these results \textbf{demonstrate a Pareto improvement by using higher-rank approximations}.
    \item We explain the competitive performance of Sketchy with observations of fast spectral decay in the moving average of Kronecker-factored gradient covariance in DL settings (Sec.~\ref{sec:spectral-analysis}).

\end{itemize}

\section{Setting and Definitions}\label{sec:setting}

\paragraph{Regret and Optimization.}

The optimization problem of training a deep neural network has a non-convex objective loss function $f$. Since finding the global optimum is computationally intractable in general, theoretical guarantees focus on convergence to an $\eps$-approximate first-order optimum: a point $x$ such that $\|\nabla f(x)\| \le \eps$. A smooth non-convex problem can be reduced to solving a series of offline convex problems \citep{ggt}. The convex sub-problems have form $f_t(x) = f(x) + c\|x - x_t\|^2$, where $c$ is a constant and $x_t$ is an iterate in the optimization process. Using online-to-batch conversion, we can translate the regret bound of an online convex optimization (OCO) \citep{hazan2016introduction} algorithm to convergence guarantees for offline optimization. For more details of this reduction, see Appendix~\ref{sec:reduction}. Therefore, non-convex optimization guarantees can be obtained from regret bounds, and we focus on the latter in this paper.

In this setting, an OCO algorithm chooses a point $x_t \in \K$ iteratively, where  $\K \subseteq \reals^d$ is a convex decision set (take $\K=\R^d$ if unconstrained). After the decision is made, the adversary reveals a convex loss function $f_t$, to which the algorithm suffers cost $f_t(x_t)$. Upon receiving the cost, the algorithm updates its decision for the next iteration. The regret for an online algorithm is given by
\begin{align*}
\regret_T=\sum_{t=1}^T f_t(x_t)-\min_{x\in\K} \sum_{t=1}^T f_t(x)\,\,.
\end{align*}

\paragraph{Sketching and the Frequent Directions Method.}
Given a stream of vectors $g_{t}\in\R^d$, $t\in[T]$, we utilize the FD sketch \citep{ghashami2016frequent} given in Alg.~\ref{alg:dfd} which maintains a low-rank approximation of the true running covariance $G_t=\sum_{s\le t}g_sg_s^\top$. At each time $t$, it maintains a matrix $B_t$ of size $d\times \ell$ whose last column is $0$ and whose square is the sketch 
$B_tB_t^\top=\bar{G}_{t}$. 
After seeing $g_t$ from the stream, we update the previous matrix using Alg.~\ref{alg:dfd}, which updates 
$B_{t+1}$ of size $d\times \ell$ whose last column remains $0$; take $B_0=0$. $B_{t+1}$ is obtained by decomposing the sum of $\bar{G}_t$ and the newly observed matrix, keeping only the top eigendirections, and reducing the eigenvalues uniformly by $\ell$-th eigenvalue. At every iteration $t$, denote $\rho_t:=\lambda_{\ell}^{(t)}$ be the removed eigenvalue from the covariance update in Alg.~\ref{alg:dfd}.

For convenience, let $\rho_{1:t}\defeq \sum_{s=1}^t \rho_s$ be the cumulative escaped mass. For a matrix $X$, we denote its  $i$-th leading eigenvalue by $\lambda_i(X)$. Let $\| \cdot \|_F$ denote the Frobenius norm of a matrix. 

\begin{algorithm}
\caption{Frequent Directions Update (\texttt{FD-update})}
\label{alg:dfd}
\begin{algorithmic}[1]
\REQUIRE Invariant that last column of $B_{t-1}$ is 0.
\ENSURE The last column of $B_t$ is 0.
\STATE Input: Previous state $\bar{G}_{t-1}=B_{t-1}B_{t-1}^\top\in\R^{d\times d}$
\STATE Input: New symmetric PSD matrix $M_t\in\R^{d\times d}$.
\STATE Eigendecompose $\bar{U}_t\diag \lambda^{(t)} \bar{U}_t^\top = \bar{G}_{t-1}+M_t$ where $\lambda^{(t)}$ contains descending eigenvalues. 
\STATE Define $U_t$ as the matrix whose columns are the first $\ell$ columns of $\bar{U}_t$, and $\lambda^{(t)}_{[1:\ell]}$ be its eigenvalues.
\STATE Update $B_t = U_t \diag\pa{\lambda^{(t)}_{[1:\ell]}-\lambda^{(t)}_{\ell}}^{1/2}$.
\algoOUTPUT $\lambda^{(t)}_\ell$, $B_tB_t^\top$.
\end{algorithmic}
\end{algorithm}

The fundamental property of FD is that applying Alg.~\ref{alg:dfd} over a stream of vectors $g_t$, with $B_0=0$, the sum of escaped mass $\rho_t=\lambda^{(t)}_\ell$ 
can be bounded by the bottom eigenvalues of $G_T$, formally given by the following lemma:
\begin{lemma}[\citet{liberty2022even}]
\label{lem:rho}
The cumulative escaped mass $\rho_{1:T}$ can be upper bounded as\[\rho_{1:T} \le \min_{k=0, \ldots, \ell-1}\frac{\sum_{i=k+1}^d \lambda_i(G_T)}{\ell-k}\le\sum_{i=\ell}^d \lambda_i(G_T)\defeq \lambda_{\ell:d}\,\, .\]
\end{lemma}
\begin{proof}
See Sec.~\ref{sec:rho-proof}.
\end{proof}

\section{Related Work}

\subsection{Spectral Analysis of DL Training}

Denote the loss function of the $i$-th example for weights $x$ as $f_i(x)$. The spectrum of the Hessian matrix $\sum_i\nabla^2f_i$ has been the subject of intensive investigation in DL \citep{sagun2016eigenvalues, sagun2017empirical, ghorbani2019investigation, sankar2021deeper} and its properties have been used to devise training methods \citep{martens2015optimizing,agarwal2017finding}.

Recent papers \citep{gur2018gradient, bakker2018understanding,xie2022rethinking} inspect the covariance matrix, $\sum_i(\nabla f_i)(\nabla f_i)^\top$. In small models, where its computation is feasible, these works identify fast spectral decay.

\citet{ggt} take advantage of this observation by using a low-rank approximation of the whole covariance matrix, based on a limited history of the gradients, $\sum_{i=t-r}^r(\nabla f_i)(\nabla f_i)^\top$. This approach still requires $r$ copies of the model gradients in memory, where typically $r$ should scale with $\beta_2^{-1}$, with $\beta_2$ the exponential moving average for second order statistics (the authors set $r=200$). Fundamentally, approximating the whole covariance matrix constrains \citet{ggt} application to small models.

In our work, we validate the decay hypothesis holds across the per-layer factored covariance matrices in several modern neural networks. For a layer's gradient matrix $G_i$ at the $i$-th example and a second moment decay term $\beta_2$, our work inspects spectral decay for $L_t=\sum_i\beta_2^{t-i}G_iG_i^\top$ and $R_t=\sum_i\beta_2^{t-i}G_i^\top G_i$; the spectral structure for these outer products is not well-documented. Furthermore, as described in Sec.~\ref{sec:shampoo}, approximating the factored covariance $L_t\otimes R_t$ requires less memory than the full covariance and explains why our method can scale to large modern architectures whereas \citet{ggt} cannot.

\subsection{Sublinear Memory Methods}

Extreme Tensoring \citep{chen2019extreme}, AdaFactor \citep{shazeer2018adafactor}, and SM3 \citep{anil2019memory} are methods that require sublinear memory relative to the number of parameters, at the other end of the memory-quality tradeoff beyond methods that rely on the diagonal of the gradient covariance such as Adam. Owing to different structural assumptions on the set of feasible preconditioners, comparison with these methods is out of scope. However, these methods may compose with our approach. One may apply Extreme Tensoring first, then sketch the resulting reshaped tensor covariances with our method to further reduce memory consumption. Similarly, an SM3-like approach which reduces the indices for each dimension to be preconditioned can be applied before Sketchy is applied to remaining indices.

Crucially, Adam, which uses linear memory for second moment representations, compares favorably in terms of quality to all of these sublinear-memory methods. \textbf{By increasing rank in factored covariance representation, Sketchy is competitive with Adam, despite sublinear memory for second moment representation.} Thus, for simplicity, we compare to only Adam, which dominates the alternative sublinear approaches in terms of quality.

\subsection{Sketching-based Approaches}

\begin{table*}
\caption{Memory-efficient adaptive gradient methods, in the OCO setting with dimension $d$ (Sec.~\ref{sec:setting}). We describe the worst-case regret bounds without exp-concavity assumptions, asymptotically, hiding logarithmic factors, treating the decision set diameter as a constant, and assume optimally-tuned hyperparameters. $\ell$ refers to the controllable preconditioner rank. Note $\tr G_T^{1/2}=\sqrt{\min_{H\in\mcH}\sum_t\norm{\nabla_t}_H^2}$ is the optimal preconditioner's regret among the class of positive semi-definite, unit-trace matrices, $\mcH$, and $G_T$ is the sum of gradient outer products. We let eigenvalues $\lambda_i = \lambda_i(G_T)$ with $\lambda_{i:j}=\sum_{m=i}^j\lambda_m$. 
}
\label{table:memory-oco}
\begin{center}
\begin{tabular}{ccc}
		\hline
		  {Reference}        & Regret (general convex)   & Memory  
\\
\hline
Full Matrix AdaGrad \citep{duchi2011adaptive} & $\tr G_T^{1/2}$ & $d^2$\\
Ada-LR \citep{krummenacher2016scalable} & $\tr G_T^{1/2}  +  \lambda_{\ell+1}^{1/2}\ell^{3/4}d^{1/4}$ & $d^2$\\ 
Ada-FD \citep{wan2021efficient}& $\Omega\pa{T^{3/4}}$\tablefootnote{The regret of Ada-FD is expressed in terms of dynamic run-time quantities which do not admit a universal bound in terms of $G_T$; we display its regret for the specific case of Observation~\ref{observation:counterexample} instead (a detailed look at its regret is given in Appendix~\ref{sec:adafd-cex}).} & $d\ell$\\
SON \citep{luo2016efficient} & $\sqrt{Td}$ & $d^2$\\
 FD-SON \citep{luo2016efficient} & $\sqrt{\ell\lambda_{\ell:d} T}$ & $d\ell$\\ %
This paper & $\tr(G_T^{1/2})+ \sqrt{d(d-\ell)\lambda_{\ell:d}}$ & $ d \ell $  \\
\hline
\end{tabular}
\end{center}
\end{table*}

Several works have explored sketching-like approximations to the gradient covariance matrix, but none provide an adaptive bound exploiting fast spectral decay in gradient covariance without additional assumptions (Tbl.~\ref{table:memory-oco}). In this section, we consider the OCO setting over dimension $d$ (Sec.~\ref{sec:setting}).

Random projection (Ada-LR) is most spiritually similar to our work \citep{krummenacher2016scalable}. Although it does not reduce memory usage, it relies on random projections to lower dimension $\ell\le d$ to reduce inverse matrix computation costs. An alternative without formal guarantees, RadaGrad, reduces memory consumption to $O(d\ell)$; however, as with all Johnson-Lindenstraus projection methods, it suffers a probabilistic failure rate scaling as $O(\ell^{-1})$ (in comparison, our method inherits FD's determinism).

Frequent Directions (FD) \citep{ghashami2016frequent, liberty2022even}, provides an alternative matrix sketching approach from the data streaming literature. As an adaptive sketch, it dominates random projection in terms of matrix recovery, and lower bounds show its memory usage is optimal in the sense that any equally-accurate approximation to an adversarially-chosen true covariance $G_T$ in operator norm constructed from the corresponding gradients must use $O(d\ell)$ bits.

In the context of exp-concave cost functions, \citet{luo2016efficient} provide an FD sketched version of Online Newton Step (ONS), FD-SON. In this setting, their approach nearly recovers classical ONS regret, up to logarithmic error in $\sum_{i=1}^{\ell-1}\lambda_i(G_T)$ and additive error in $\sum_{i=\ell}^d\lambda_i(G_T)$. However, without the exp-concave assumption, FD-SON falls back to a gradient-descent-like default regret of $O\pa{\lambda_{\ell:d}\sqrt{T}}$, which can be $\Omega(T)$ without spectral decay. In the context of linear bandits, \citet{chen2021efficient} uses FD for memory reduction. The resulting algorithm, SCFD, is similar to Alg.~\ref{alg:AdaGrad}, but SCFD lacks a projection step and does not handle general domains $\mcK$. We emphasize that our contribution is in the novelty of our regret analysis for general OCO settings, in contrast to linear bandits.

The main prior work exploring FD for general online convex settings, \citet{wan2021efficient}, extends the FD-SON approach by adding a fixed diagonal perturbation $\delta I$ to an FD-based preconditioner, in Ada-FD. However, this approach does not achieve $\sqrt{T}$ regret even in a non-adversarial setting with stochastic linear cost functions (Observation~\ref{observation:counterexample}), where learning rate and $\delta$ are tuned. Dynamically changing diagonal regularization is essential for worst-case $O\pa{\sqrt{T}}$ performance.
\begin{observation}\label{observation:counterexample} 
Suppose we receive linear cost functions $f_t(x)=\inner{x, g_t}$, where $g_t \in \R^d$ is a random vector drawn iid from any distribution over $r\le d$ orthonormal vectors $W$. For any sketch size $\ell\le r$, the bound on the expected regret of Ada-FD is $\Omega(T^{3/4})$.
\end{observation}
\begin{proof}
See Sec.~\ref{sec:adafd-cex}.
\end{proof}
\citet{wan2021efficient} remark Ada-FD has $\sqrt{T}$ regret when $G_T$ is low rank with rank below $k$, so $d-k$ of its eigenvalues are precisely zero. However, this setting does not require any sketching in the first place. By tracking the column space of observed gradients (e.g., with a reduced QR decomposition, rank-1-updated every step), the full matrix AdaGrad algorithm can be perfectly recovered without using more than $O(dk)$ memory.

In concrete convex examples, Sketchy compares favorably to these approaches (Appendix~\ref{sec:online-cvx-ex}).

\subsection{Shampoo}\label{sec:shampoo}

\begin{figure}[]
\begin{center}
\includegraphics[width=0.9\textwidth]{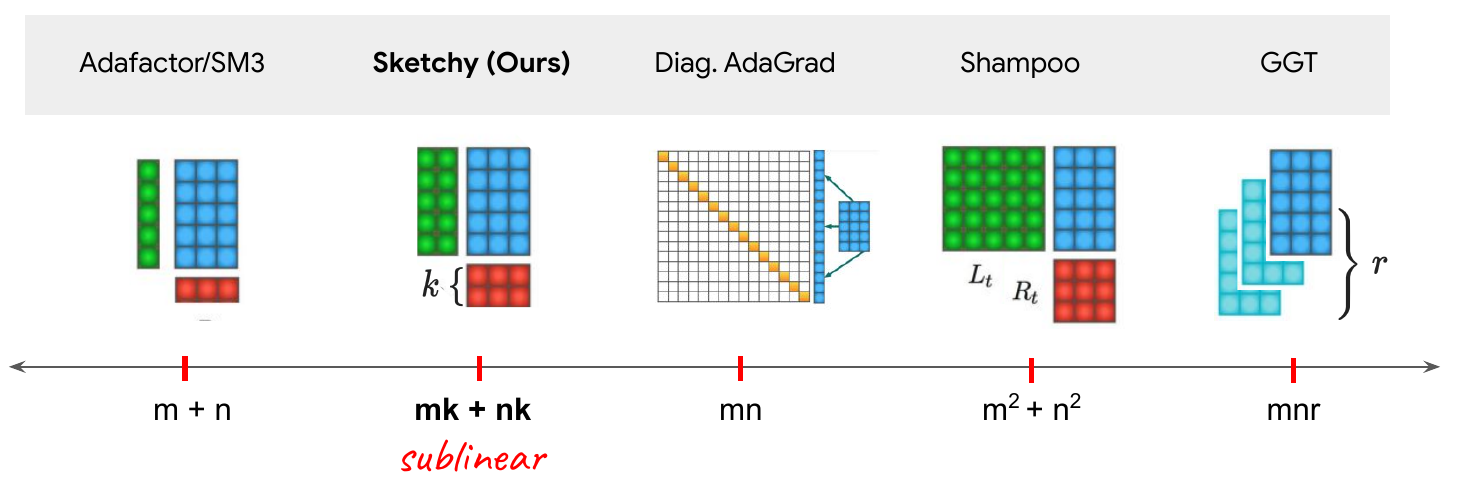}
\end{center}
\vskip -0.1in
\caption{Asymptotic memory consumption for representing gradient covariance in adaptive regularization approaches for a single matrix parameter of size $n\times m$. Here, $r$ refers to the GGT history buffer size and $k$ to the approximation rank of FD (both typically set to hundreds). Past sketching approaches like Ada-FD and Radagrad take memory similar to GGT, with $r$ being sketch size; these are all asymptotically superlinear. This figure demonstrates optimizer memory usage in the theoretical OCO setting; in practice for deep learning workloads there are additive $O(mn)$ factors for momentum, the parameters themselves, and grafting parameters for Shampoo and Sketchy.} 
\vskip -0.2in
\label{table:memory}
\end{figure}

Perhaps our most compelling application is reducing the memory of Shampoo \citep{gupta2018shampoo, anil2020scalable, anil2022factory}. Shampoo is an adaptive preconditioning method that takes into account the structure of the parameter space, and thus is more efficient than full matrix AdaGrad. For example, if the parameter is a weight matrix $W$ of size $m\times n$, AdaGrad treats the matrix-shaped parameters as a vector of size $mn$, and the preconditioner has size $m^2n^2$; Shampoo instead has left and right preconditioners $L, R$ of size $n\times n$ and $m\times m$, respectively, with the preconditioned update $L^{-1/4}WR^{-1/4}$. Write $\bvec{W}$ as the vectorized weights, then it is equivalent to $(L\otimes R)\bvec W=\bvec{LWR}$, where $\otimes$ denotes the Kronecker product. Figure \ref{table:memory} illustrates the updates of AdaGrad and Shampoo, where AdaGrad update uses the entire matrix instead of the diagonal, which is shown in the figure. In other words, Shampoo uses a Kronecker-factored preconditioner, and the factorization preserves the matrix structure of the parameters. Since in DL optimization parameters often have matrix structure, Shampoo has strong empirical performance, and has improved upon state-of-the-art results in large-scale tasks such as language modeling with BERT-Large and image classification on ImageNet in \citep{anil2020scalable}. 

Figure~\ref{table:memory} elaborates why the composition of FD and Shampoo is essential to avoid memory consumption asymptotically greater than parameter count for approximate full matrix regularization.

However, Shampoo memory costs may still be prohibitive for rectangular weight matrices. In BERT-Large \citep{devlin2019bert}, most parameters are in the feed-forward network layers, which consist of $4096\times 1024$ dense kernels; other transformers follow similar narrow-to-wide patterns. For large models, occupying even $4\times$ memory for the left preconditioner can frequently result in OOM in memory-constrained settings; this was in fact one of the practical motivations for our proposed approach.

\citet{anil2020scalable} introduces two workarounds for the problem of rectangular matrices based on limiting covariance modelling. Furthermore, both approximations can be applied to our method, so we do not compare against them. First, the authors propose Blocked Shampoo, which views each weight matrix $W$ of shape $m\times n$ as $mn/b^2$ blocks of size $b\times b$ for some block size $b < \min(m,n)$ (in the limit $b=1$, this recovers diagonal AdaGrad). This approach is dependent on the ordering of neurons in hidden layers. Another approximation relies on only one-sided covariance upper bounds, $L_t\otimes I$ or $I\otimes R_t$. Note, however, that the one-sided approximation doesn't help with vector parameters, such as those that appear for the bias terms in dense layers or layer norms \citep{ba2016layer}. For 3D weights, such as those which appear in homogeneous Mixtures of Experts \citep{shazeer2017outrageously}, blocking increases memory consumption. These approaches do not take into account the fast decay of the preconditioner's spectrum, which is the focus of our work.

\section{Algorithms and Main Theorems}
In this section, we introduce the adaptation of Frequent Directions (FD) to AdaGrad (Sec.~\ref{sec:fd-AdaGrad}) and Shampoo (Sec.~\ref{sec:fd-shampoo}), the corresponding algorithms and regret guarantees. Additionally, in Sec.~\ref{sec:exp-weighted-fd}, we modify FD to support exponential moving averages. 

The main technical novelty in incorporating Alg.~\ref{alg:dfd} to AdaGrad (Alg.~\ref{alg:AdaGrad}) and Shampoo (Alg.~\ref{alg:shampoo}) is the construction of preconditioning matrices with FD-sketched matrices compensated by the cumulative escaped masses. The insight of such construction lies in the observation that while the FD sketch lower bounds the full preconditioning matrix, the FD sketch compensated with the cumulative escaped masses upper bounds the full preconditioning matrix, as demonstrated in Lemma~\ref{lem:fd-update} for AdaGrad and Lemma~\ref{lem:fd-update-shampoo} for Shampoo. The regret guarantee for AdaGrad (\citep{duchi2011adaptive}) and Shampoo (\citep{gupta2018shampoo}) directly depends on the trace of the preconditioning matrices, therefore obtaining upper and lower bounds on the preconditioning matrices allows explicit additive dependence on the cumulative escaped mass. We expect this approach to be reusable for alternative approximation schemes.

\subsection{FD for AdaGrad}
\label{sec:fd-AdaGrad}
Our main algorithm in this section is Alg.~\ref{alg:AdaGrad} run with FD (Alg.~\ref{alg:dfd}) as the sketching method. $\tilde{G}_t^{-1/2}$ in Alg.~\ref{alg:AdaGrad} denotes the Moore-Penrose pseudoinverse of the matrix $\tilde{G}_t^{1/2}$.
Our main algorithm, Sketchy AdaGrad, in this section exploits the FD approach outlined in Alg.~\ref{alg:dfd} as the sketching method in AdaGrad. In particular, at every time step, we pass the newly received subgradient $g_t$ into Alg.~\ref{alg:dfd}, which updates and maintains a low-rank sketch $\bar{G}_{t}$ of the AdaGrad preconditioning matrix $G_t$. We keep track of the cumulative escaped mass $\rho_{1:t}$, which we add back to the low-rank sketch to create the Sketchy preconditioner $\tilde{G}_{t}$, with which we perform the regular AdaGrad descent and projection.
\begin{algorithm}
\begin{algorithmic}[1]
\caption{Sketchy AdaGrad (\texttt{S-AdaGrad})}
\label{alg:AdaGrad}
\STATE Input: constraint set $\K$, step size $\eta$, time horizon $T$. 
\STATE Initialize $x_1\in\K$, $\bar{G}_0=\tilde{G}_0=0$.
\FOR {$t=1,\ldots,T$}
\STATE Play $x_t$, receive $g_t\in\partial f_t(x_t)$, suffer cost $f_t(x_t)$.
\STATE Sketch $(\rho_t,\bar{G}_t)=\texttt{FD-update}(\bar{G}_{t-1},g_tg_t^\top)$.
\STATE Update $\tilde{G}_t=\bar{G}_t+\rho_{1:t}I$, $y_{t+1}=x_t-\eta \tilde{G}_t^{-1/2}g_t$, and $x_{t+1}=\underset{x\in\K}{\argmin} \ \|y_{t+1}-x\|_{\tilde{G}_t^{1/2}}^2$.
\ENDFOR
\end{algorithmic}
\end{algorithm}
\begin{theorem}\label{thm:add} 
Define $\Omega_\ell = \min_{k<\ell}(\ell-k)^{-1}\sum_{i=k+1}^d\lambda_i(G_T)$, then with $\eta=\frac{D}{\sqrt{2}}$, Alg.~\ref{alg:AdaGrad} guarantees the following additive regret bound:
\[
\regret_T(\texttt{S-AdaGrad})\le D\left(\sqrt{2}\tr G_T^{1/2}+d\sqrt{\frac{\Omega_\ell}{2}}\right)\,\, ,
\ignore{\frac{d\sqrt{\Omega_\ell}+\tr G_T^{1/2}}{2\eta} D^2+\eta \tr G_T^{1/2},}
\]
where $D$ is the diameter of the constraint set $\K$ if $\K$ is bounded and $\max_{t\in[T]}\|x_t-x^*\|_2$ otherwise. 
\end{theorem}
\begin{proof}
See Sec.~\ref{sec:proof-fd-ada}.
\end{proof}

Notably in Thm.~\ref{thm:add}, $\regret_T=O\pa{\sqrt{T}}$ and the lower eigenvalue dependence $\Omega_\ell$ is additive.

\begin{corollary}
\label{thm:improved}
We can improve Theorem~\ref{thm:add} slightly to 
\[
\regret_T(\texttt{S-AdaGrad})\le D\left(\sqrt{2}\tr G_T^{1/2}+\sqrt{\frac{d(d-\ell)\Omega_\ell} {2}}\right)\,\,.
\]
\end{corollary}
\begin{proof}
See Sec.~\ref{sec:proof-thm-improved}.
\end{proof}

The regret bound above holds under the optimal tuning of the learning rate, which depends on problem quantities that can be unknown a priori. 
It is possible to design  a parameter-free variant of Alg.~\ref{alg:AdaGrad} by using the norm $\|x\|_t = (x^\top (\tilde{G}_t + I)^{1/2}x)^{1/2}$ in the projection step of Alg.~\ref{alg:AdaGrad}, as seen in \citep{param_free_AdaGrad}.

\subsection{FD for Shampoo}
In this section, we adapt \texttt{FD-update} to Shampoo \cite{gupta2018shampoo}. For simplicity, we optimize over $\R^{m\times n}$ in Alg.~\ref{alg:shampoo}; projection may be handled as in Alg.~\ref{alg:AdaGrad}.
Similar to Sketchy AdaGrad, Sketchy Shampoo uses the FD approach outlined in Alg.~\ref{alg:dfd} to sketch the left and right preconditioning matrices for Shampoo. In particular, we maintain two parallel sketching streams using Alg.~\ref{alg:dfd} to produce sketches $\bar{L}_{t}, \bar{R}_{t}$ for the left and right preconditioning matrices. We keep track of the cumulative escaped masses $\rho_{1:t}^L$ and $\rho_{1:t}^R$ from sketching the left and right preconditioning matrices, respectively, and compensate the cumulative escaped mass to create the left and right Sketchy preconditioning matrices $\tilde{L}_{t}, \tilde{R}_{t}$.

\label{sec:fd-shampoo}
\begin{algorithm}
\begin{algorithmic}[1]
\caption{Sketchy Shampoo (\texttt{S-Shampoo})}
\label{alg:shampoo}
\STATE Input: step size $\eta$, time horizon $T$. 
\STATE Initialize $X_0 = 0_{m\times n}$, $\tilde{L}_0=\eps I_m$, $\tilde{R}_0 = \eps I_n$, $\bar{L}_0=0_m$, $\bar{R}_0 =0_n$.
\FOR {$t=1,\ldots,T$}
\STATE Play $X_t$, suffer $f_t(X_t)$, receive $G_t\in \partial f_t(X_t)$.
\STATE Sketch $(\rho_t^L, \bar{L}_t)=\texttt{FD-update}(\bar{L}_{t-1},G_tG_t^\top)$, $(\rho_t^R, \bar{R}_t)=\texttt{FD-update}(\bar{R}_{t-1},G_t^\top G_t)$.
\STATE Update $\tilde{L}_t = \bar{L}_t + \rho^L_{1:t} I_m$, $\tilde{R}_t = \bar{R}_t + \rho^R_{1:t} I_n$ and $X_{t+1}=X_t-\eta \tilde{L}_t^{-1/4}G_t\tilde{R}_t^{-1/4}$. 
\ENDFOR
\end{algorithmic}
\end{algorithm}

Denote $
L_T\defeq \sum_{t=1}^TG_tG_t^\top + \eps I$ and $R_T\defeq \sum_{t=1}^TG_t^\top G_t + \eps I$.

\begin{theorem}\label{thm:add_shampoo} Suppose $G_1, \ldots G_T$ have rank at most $r$. Then Alg.~\ref{alg:shampoo} run with $\eta = D/\sqrt{2r}$ guarantees the following regret bound:
\begin{align*}
\regret_T(\texttt{S-Shampoo})\le \sqrt{2r} D\pa{\tr(L_T^{1/4}) + m\Omega_{L, \ell}^{1/4}}
\pa{\tr(R_T^{1/4}) + n\Omega_{R, \ell}^{1/4}}\,\,,
\end{align*}
where $D=\max_{t\in[T]} \|X_t-X^*\|_{F}$ and $\Omega_{L,\ell},\Omega_{R,\ell}$ are analogous bounds for $\rho_{1:T}^L, \rho_{1:T}^R$ from Lem.~\ref{lem:rho}.
\end{theorem}
\begin{proof}
See Sec.~\ref{sec:proof-shampoo}.
\end{proof}
Bounds may be improved analogous to Cor.~\ref{thm:improved} for Alg.~\ref{alg:shampoo}, but we omit the similar statement due to space.

\subsection{Exponentially Weighted FD}
\label{sec:exp-weighted-fd}
This section discusses the modification of Alg.~\ref{alg:dfd} to support exponential moving averages. Early in algorithm development, we noticed that attempting to approximate the unweighted sum of factored gradient covariances $\sum_t G_tG_t^\top$ and $\sum_t G_t^\top G_t$ with FD tended to an estimate of covariance that was roughly $0$, creating numerical instabilities. Note that FD guarantee (Lem.~\ref{lem:rho}) still holds---but the error term $\rho_{1:T}$ becomes greater than $\norm{G_T}$, resulting in a vacuous bound due to lack of spectral decay.

Indeed, Fig.~\ref{fig:intrinsic} motivating this work only confirmed that the exponential moving average $L_t(\beta_2)=\sum_t\beta_2^{T-t}G_tG_t^\top$ exhibits fast spectral decay (and analogously for $R_t$). Luckily, thanks to the recursion $L_{t+1}(\beta_2)=\beta_2 L_t+ G_{t+1}G_{t+1}^\top$, the FD sketch may easily be adopted for this setting.

\begin{observation}\label{observation:exp-fd} 
Given a stream $g_t$ of vectors for $t\in[T]$, sketch size $\ell$, updates $(\rho_t^{(\beta_2)},\bar{G}_t^{(\beta_2)})=\texttt{FD-update}(\beta_2\bar{G}_{t-1}^{(\beta_2)},g_tg_t^{\top})$, and $G_T^{(\beta_2)}=\sum_{t=1}^T\beta_2^{T-t}g_tg_t^\top$,  we have
\[
\norm{\bar{G}_T^{(\beta_2)}-G_T^{(\beta_2)}}\le \rho_{1:T}^{(\beta_2)}\le
\min_{k<\ell}\frac{\sum_{i=k+1}^d \lambda_i\pa{G_T^{(\beta_2)}}}{\ell-k}\,\,.
\]
\end{observation}

\section{Experiments}

We investigate how much of Shampoo's quality our low-memory approach can recover (Sec.~\ref{sec:generalization-benchmarks}) and whether the factored covariance exhibits spectral decay amenable to sketching (Sec.~\ref{sec:spectral-analysis}). We relegate convex examples to Appendix~\ref{sec:online-cvx-ex} due to space constraints, which check that Sketchy compares favorably to related work in a setting directly related to the theory.

\subsection{Deep Neural Networks}\label{sec:generalization-benchmarks}

We evaluate the effectiveness of \texttt{S-Shampoo} as a practical second-order algorithm for training networks, including
\begin{itemize}
    \item ResNet-50 \citep{he2016deep} for ImageNet image classification task of ImageNet \citep{ILSVRC15} with random cropping and flipping augmentations.
    \item A 16-layer Conformer model \citep{gulati2020conformer} for the audio transcription task, Librispeech \citep{panayotov2015librispeech}.
    \item A GNN with 5 message-passing steps \citep{battaglia2018relational} on \texttt{ogbg-molpcba} \citep{DBLP:conf/nips/HuFZDRLCL20}, which classifies structural properties of graphically encoded molecule inputs.
\end{itemize}

Our FD variant of Shampoo introduces only one new hyperparameter, the rank $\ell$, which we do not tune, but set to $\ell = 256$, which translates to $4\times$ memory savings for Shampoo blocks of size $1024$ for the accumulators. Shampoo, Adam, and the underlying architectures introduce their own hyperparameters. \texttt{S-Shampoo} inherits those of Shampoo. We tune only common parameters between the three optimizers with the same budgets, selecting based on validation set accuracy. In the ImageNet case, we evaluate final test set performance using ImageNet v2 \citep{recht2019imagenet}, as the ImageNet test set is unavailable. Additinoal training information is available in Appendix~\ref{sec:appendix-training-general}.

\begin{figure}
\centering
\includegraphics[width=\columnwidth]{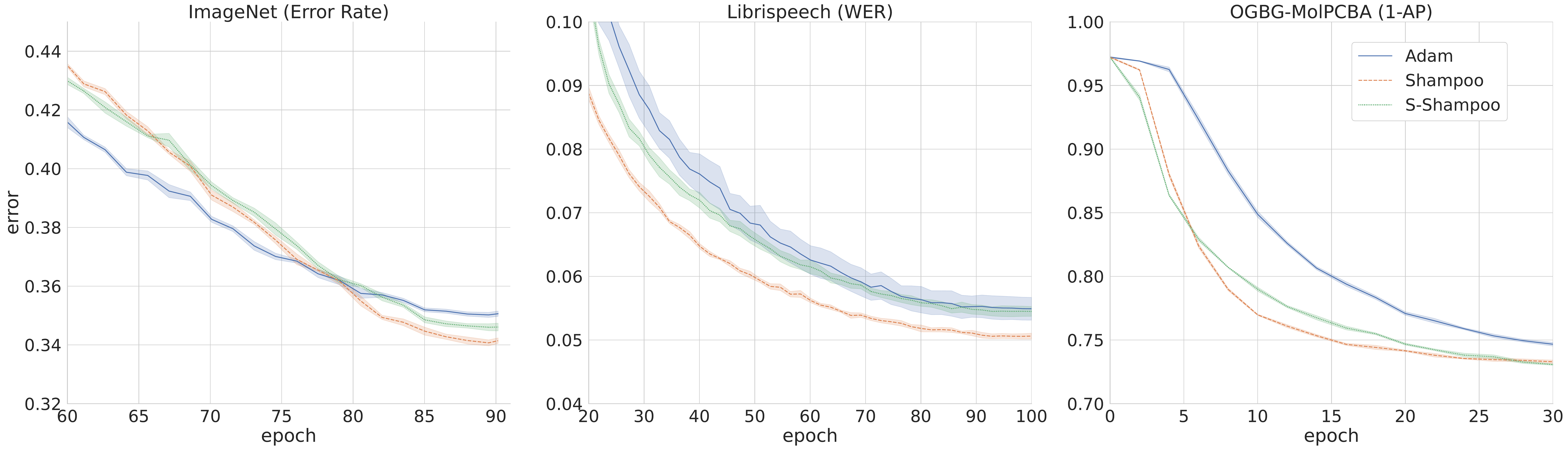}
\caption{Test metrics are classification error rate for top-1 accuracy for ImageNet v2, word error rate for Librispeech, and one minus average precision for OGBG-MolPCBA. We plot the mean of 5 random seeds, with 1.96 times the standard error as error bars. For readers familiar with ImageNet v1, final validation accuracy for Shampoo was 77.69\% (0.03\%), \texttt{S-Shampoo} having 77.18\% (0.04\%), and Adam having 76.76\% (0.03\%), but we emphasize that due to tuning, the test set performance pictured above should be of primary concern.}\label{fig:dnncurves}
\end{figure}

As Fig.~\ref{fig:dnncurves} demonstrates, the second-order information leveraged by Shampoo results in improvements over Adam, a first-order method. Our method performs at least as well as Adam in all cases, \textbf{despite using asympotically less memory to represent covariance} (as Fig. \ref{table:memory} shows, Adam uses $O(mn)$ for a rectangular weight matrix's diagonal accumulators, whereas \texttt{S-Shampoo} uses $O(mk+nk)$). In the GNN case (OBGB-MolPCBA), Shampoo does not perform as well; its training curves indicate overfitting, but we note that \texttt{S-Shampoo} was less susceptible to overfit like Adam.

\subsection{Spectral Analysis}\label{sec:spectral-analysis}

To explain Sketchy's strong performance in Sec.~\ref{sec:generalization-benchmarks}, we inspect the exponential moving average of Kronecker-factored gradient covariance for fast spectral decay. We find that this is indeed the case in practice, so Sketchy's low-rank plus diagonal covariance is representative of true training statistics.

For all our architectures, we tune Shampoo and extract the intermediate gradient covariances over the course of training. To make our curves comparable across architectures, we fix the parameter for the second moment, $\beta_2=0.999$ for these runs. Furthermore, ResNet-50 has a few parameters with dimension 2048, but the largest dimension for any parameter from the other two architectures is 1024, so we use the Blocked Shampoo variant discussed in Sec.~\ref{sec:shampoo} with block size 1024. In other words, weights containing a dimension 2048 are split into two. We tune other Shampoo parameters for each architecture, and plot statistics of Kronecker factors $L_t=\sum_i\beta_2^{t-i}G_tG_t^\top$ and $R_t=\sum_i\beta_2^{t-i}G_t^\top G_t$.

\begin{figure}
\begin{center}
\begin{subfigure}
\centering
\includegraphics[width=0.4\linewidth]{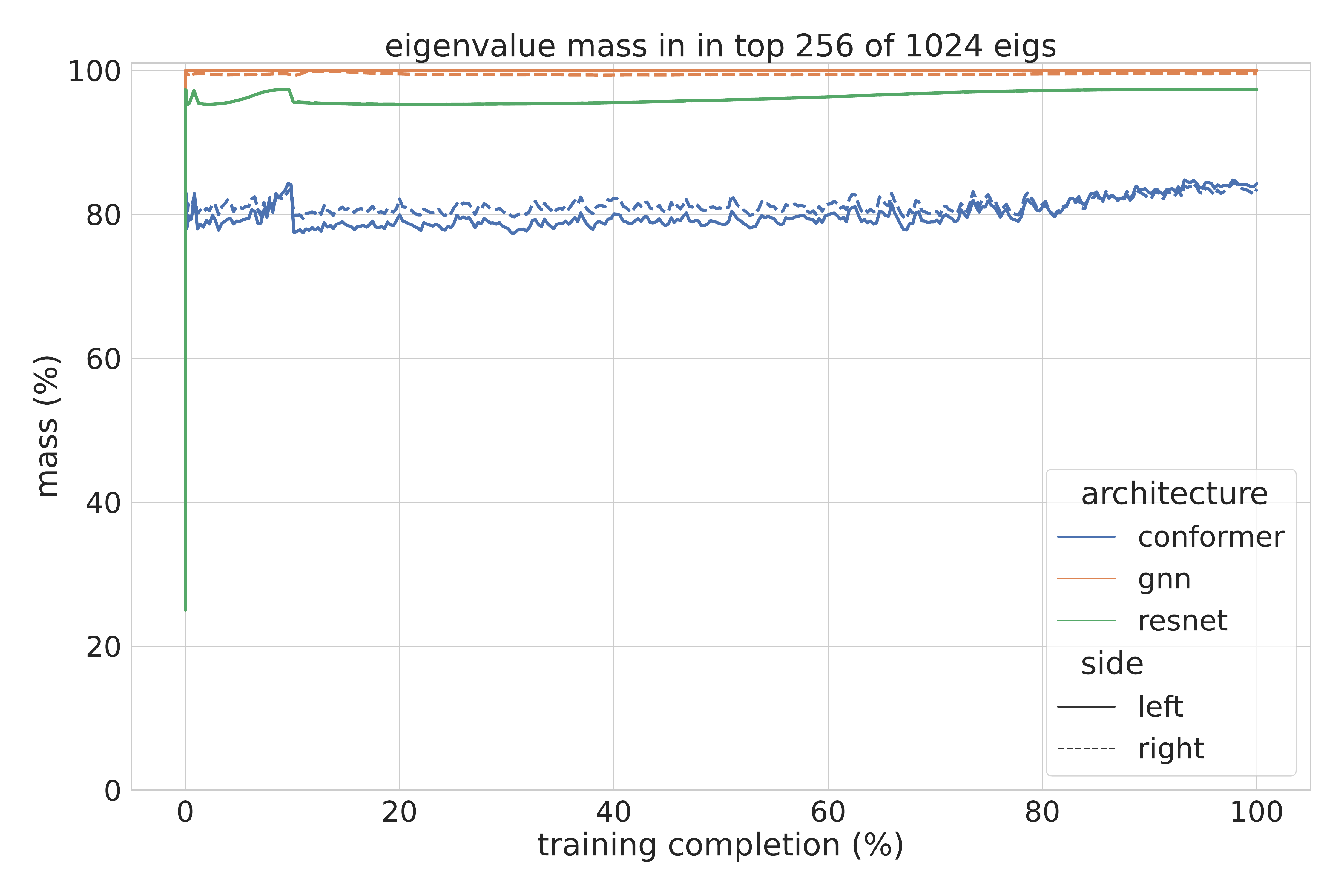}
\end{subfigure}%
\hspace{1cm}
\begin{subfigure}
\centering
\includegraphics[width=0.4\linewidth]{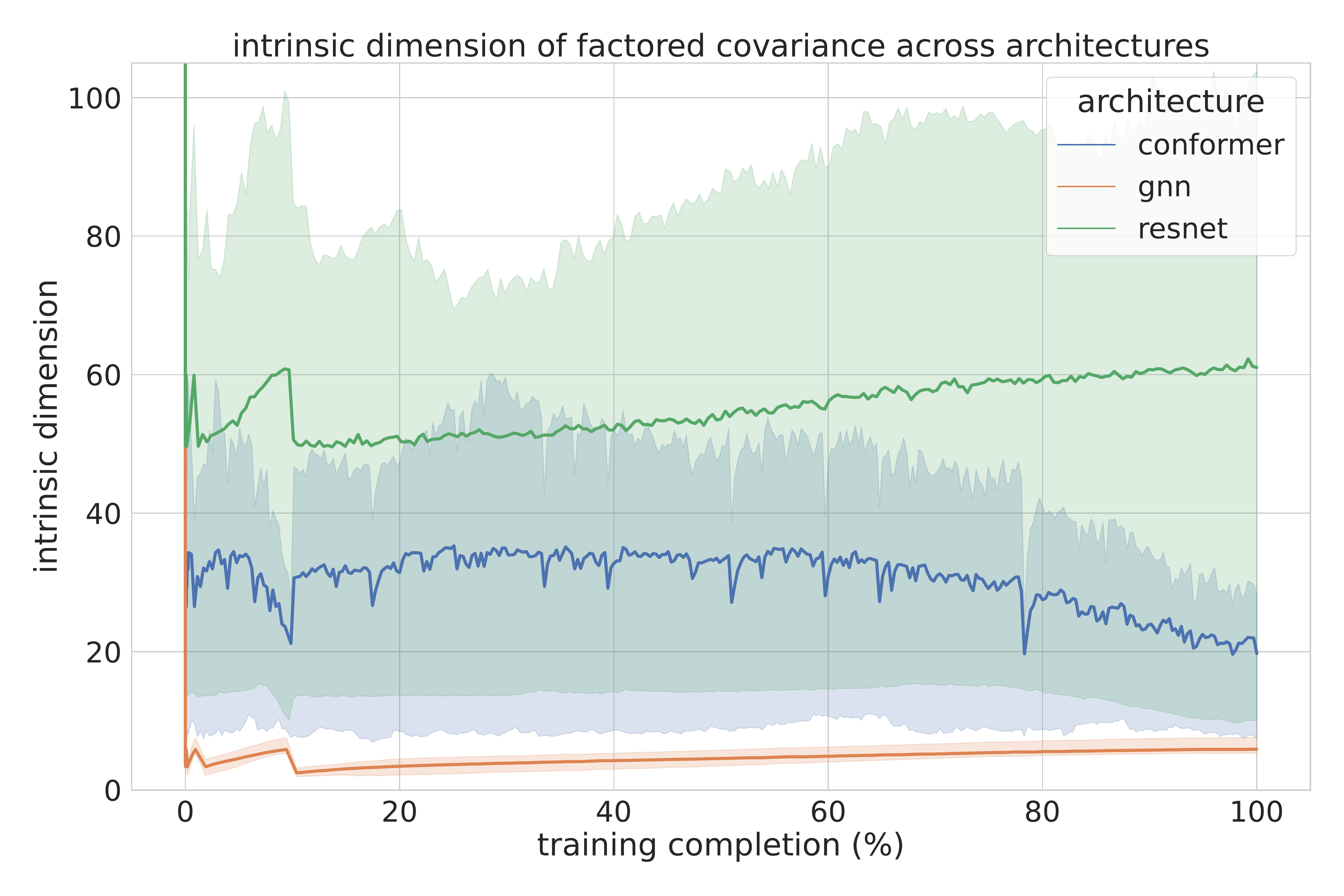}
\end{subfigure}
\end{center}
\caption{Two measures of spectral decay. As described in Sec.~\ref{sec:spectral-analysis}, we demonstrate spectral decay of $L_t$ and $R_t$ covariance factors of shape $1024\times 1024$. On the left, we take the factors $C$ from the first layer of each network and plot the proportion of spectral mass captured by its top $256$ eigenvalues throughout training, i.e., $\sum_{i=1}^{256}\lambda_i(C)/\sum_{i=1}^{1024}\lambda_i(C)$. On the right, we plot a continuous measure of eigenvalue concentration, the intrinsic dimension $\tr C / \lambda_{\max}(C)$, typically $10\times$ smaller than nominal dimension. We plot the average across all covariances for each network's weight's dimensions, with the shaded regions capturing the interquartile range.}\label{fig:intrinsic}
\end{figure}

In Fig.~\ref{fig:intrinsic}, we plot the intrinsic dimension of Kronecker covariance factors over training for our three settings. The intrinsic dimension determines the rate at which empirical covariance estimates concentrate to their expectation, rather than a random vector's actual dimension, up to logarithmic factors (\citet{vershynin2018high}, Remark 5.6.3). Despite actual dimensionality being over 1024, intrinsic dimension across all architectures stays below 105. A conspicuous phase shift 10\% of the way through training may be the result of a change from linear learning rate warmup to a learning rate decay, starting at roughly 5\% of the way into training.

Given $\beta_2=0.999$, we emphasize that the behavior in Fig.~\ref{fig:intrinsic} is an emergent property of DL training. Though surely a lower $\beta_2$ would naturally result in lower intrinsic dimension (which can still be taken advantage of by Alg.~\ref{alg:AdaGrad} and \ref{alg:shampoo}), we would still expect higher intrinsic dimension if covariances were near-isometries. If we observe some large number $n=10000$ draws $x_i$ of $1024\times d$ matrices with iid $N(0, 1)$ entries, then numerical experiments show that the average intrinsic dimension of $\sum_{i=0}^{n-1}\beta_2^ix_ix_i^\top$ is $324.63$ ($0.52$) and $862.13$ ($0.25$) for $d=1, 64$, respectively, with parenthesized numbers denoting standard error across 20 trials. Values generated this way are larger than the average intrinsic dimension of roughly 10, 30, 50 observed in Fig.~\ref{fig:intrinsic}.

\section{Discussion}\label{sec:discussion}

Up to spectral error, Alg.~\ref{alg:AdaGrad} achieves full-matrix AdaGrad regret despite approximating the \emph{smallest} part of the spectrum of $G_t^{-1/2}$ at each step. Remarkably, these eigenvectors correspond to the \emph{most} easily discernible signals of the covariance for the stream $g_t$. This apparent (and fortuitous) coincidence is resolved by considering the covariance of $\tilde{G}_t^{-1/2}g_t$: whitening the gradient to facilitate optimization best reflects on regret; as a result, approximating top eigenvectors of $G_T$ helps more than the bottom ones.

Our initial implementation focused on correctness rather than physical speed or memory reduction. Engineering optimizers competitive with existing industrial-strength implementations of Adam and Shampoo was out of scope. In implementing FD, we performed updates via the factored SVD of $[\beta_2^{1/2} B_t; G_t]$ rather than the eigendecomposition depicted in Alg.~\ref{alg:dfd}; this avoids squaring, which is unavoidable in Shampoo. For speed, Shampoo subsamples gradients for its covariance estimation and updates its inverse matrix roots intermittently, every fixed number of steps. A tuning script provided by \citet{anil2020scalable} included gradients from every step, but updated roots every 10 steps. Since FD does not separate sampling from its computation of estimated covariance eigendecomposition, we took the more difficult setting for \texttt{S-Shampoo}, only allowing it to simultaneously observe every 10\textsuperscript{th} gradient and update its covariance inverse roots (see Appendix~\ref{sec:step-skipping} for a theoretical justification).

Though step-skipping makes Shampoo and \texttt{S-Shampoo} tractable, future work may explore further speedups: since FD only requires the top $\ell$ eigenvalues, iterative Lanczos-like routines which are accelerator-friendly, such as LOBPCG \citep{knyazev2001toward}, may allow incremental updates to $\tilde{G}_t^{-1/2}$ in factored form with only a few matrix multiplies, \texttt{S-Shampoo} may be able to update more frequently than its non-sketched counterpart, further improving quality.
\section{Conclusion}

In this work, we address a gap in the OCO literature for low-memory optimization with the novel Alg.~\ref{alg:AdaGrad} and demonstrate its relevance to practical non-convex problems such as neural net training (Sec.~\ref{sec:generalization-benchmarks}) by leveraging a new observation about gradient covariance (Sec.~\ref{sec:spectral-analysis}).

The growing disparity between compute capability and memory bandwidth \citep{jouppi2021ten} underscores the need for further research in this direction. Further, large-batch settings reduce the performance gap between first and Shampoo-based second order methods, since the batch-size independent runtime of the optimizer is amortized per example used for the gradient calculation. Even in performing experiments for this work, we would frequently find that faster accelerators were unavailable, but many previous-generation ones were, encouraging us to leverage data-parallel training. For datasets such as Imagenet, we notice the advantage of second order methods in dealing with large batches even at relatively modest sizes, such as 1024; many works on explore several larger multiples of this \citep{keskar2016large}.

Potential for future work includes numerical methods outlined in the previous section as well optimizing the rank $\ell$ across the many tensors in a network, as the spread in Fig.~\ref{fig:intrinsic} highlights the large variance in covariance intrinsic dimension. Furthermore, the inductive biases conferred by the minima which different-rank representations of curvature reach may have problem-dependent generalization implications, a question which we leave for future work. For a comparison of full rank preconditioning's effect versus first-order minima, see \citet{amari2020does}.

\newpage
\bibliographystyle{unsrtnat}
\bibliography{main.bib}
\newpage
\appendix

\section{Online Convex Examples}\label{sec:online-cvx-ex}

In this section, we evaluate the performance of \texttt{S-Adagrad} in the classical online convex optimization setting.

We consider three different datasets, summarized in Tbl.~\ref{tbl:oco-example-counts}. For each dataset, we evaluate several works based on the frequent directions sketch, Ada-FD \cite{wan2021efficient}, FD-SON \cite{luo2016efficient}, RFD-SON \cite{luo2019robust}. As baselines we also add online gradient descent (OGD) and diagonal Adagrad \cite{duchi2011adaptive}. For each dataset, we consider the online convex loss of a logistic binary classification loss over a linear learner on the features in each dataset augmented with a constant feature for the intercept.

\begin{table}
\centering
\caption{Dataset statistics for simple online convex examples. We stream through each example, providing a logistic loss over the binary label and linear predictor generated by an OCO learner. The feature count includes an all-constant intercept column as well. Datasets were retrieved from \citet{chang2011libsvm}.}\label{tbl:oco-example-counts}
\begin{tabular}{lcc}
\hline
Dataset Name & Number of Examples & Number of Features \\
\hline
\texttt{gisette\_scale} & 6000 & 5001 \\
\texttt{a9a} & 32561 & 124 \\
\texttt{cifar10} & 50000 & 3073 \\
\hline
\end{tabular}
\end{table}

For methods which require a nonzero diagonal initial regularizer $\delta I$, namely FD-SON and Ada-FD, we tune $\delta$ in $10^{-6},10^{-5},\cdots 1$ and the learning rate $\eta$ over the same range as well, for 49 hyperparameters total. For methods which have $\delta=0$ (Adagrad, OGD, S-Adagrad, RFD-SON), we instead tune $\eta$ on 49 points spaced evenly on the same logarithmic scale, $[10^{-6}, 1]$ to fairly allocate hyperparameter training budget. Note that the variant of RFD-SON, $\mathrm{RFD}_0$, which sets $\delta=0$ is the main variant evaluated by \citet{luo2019robust}. The sketch size was fixed to be 10 throughout.

\begin{table}
\centering
\caption{Average cumulative online loss across datasets and algorithms, ranked from lowest (1st place) to highest (6th place). Our proposal, \texttt{S-Adagrad}, is in bold.}\label{tbl:oco-placements}
\begin{tabular}{llllllll}
\toprule
    & Place &          1 &          2 &         3 &         4 &         5 &            6 \\
\midrule
\texttt{cifar10} & Alg. &    RFD-SON & \textbf{ S-Adagrad }&   Adagrad &       OGD &    Ada-FD &       FD-SON \\
    & Loss &   0.297 &   0.297 &  0.303 &  0.308 &   2.999 &  $6 \times 10^6$ \\
\texttt{gisette} & Alg. & \textbf{S-Adagrad} &    RFD-SON &    Ada-FD &   Adagrad &       OGD &       FD-SON \\
    & Loss &   0.158 &   0.167 &  0.196 &  0.209 &     0.224 &      2.432 \\
\texttt{a9a} & Alg. &    Adagrad &  \textbf{S-Adagrad} &       OGD &   RFD-SON &    Ada-FD &       FD-SON \\
    & Loss &    0.332 &    0.333 &  0.335 &  0.335 &  0.354 &     0.539 \\
\bottomrule
\end{tabular}
\end{table}

We sort and display cumulative average regret in Tbl.~\ref{tbl:oco-placements} and Fig.~\ref{fig:ococurves}. S-Adagrad is the only method to consistently place among the top three across all datasets. We suspect that a combination of allowing $\delta=0$ (removing inductive bias about regularization; notice in Tbl.~\ref{tbl:oco-placements} that Ada-FD and FD-SON, the two methods with $\delta>0$, routinely place last) and the ability to deal with the effectively zero exp-concavity constant of the logistic loss \citep{hazan2014logistic} explain S-Adagrad's performance.

\begin{figure}
\centering
\includegraphics[width=0.8\columnwidth]{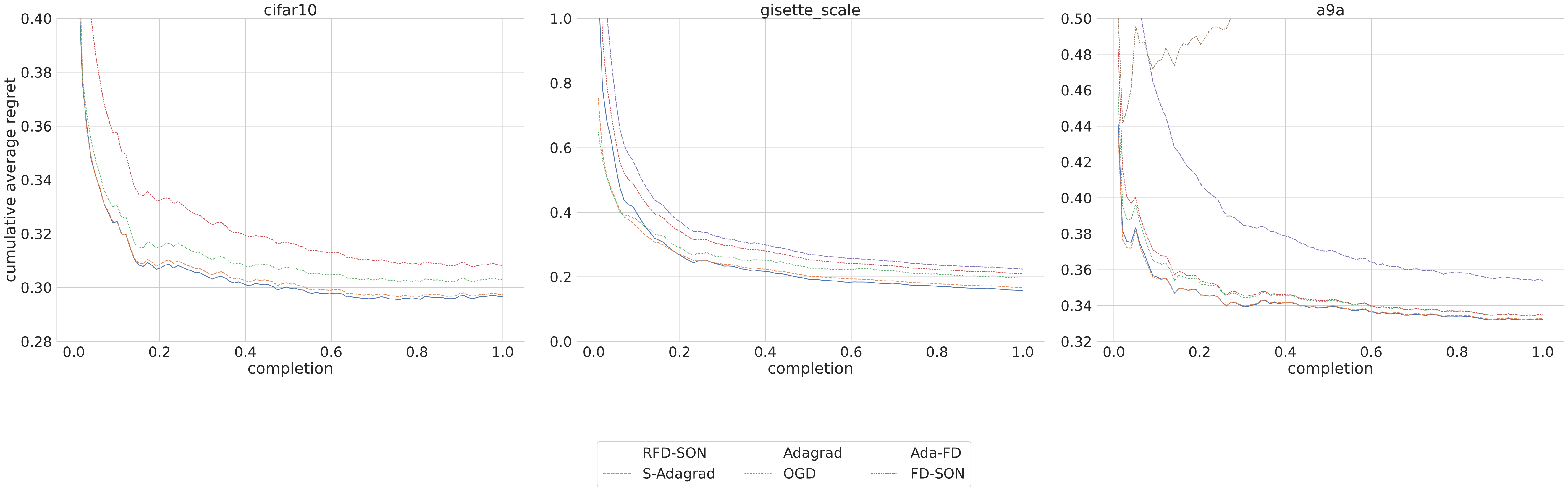}
\caption{Cumulative average regret curves for logistic loss as a function of percent of dataset completion in a single online pass, resummarizing Tbl.~\ref{tbl:oco-placements} visually.}\label{fig:ococurves}
\end{figure}

\section{Proof details}\label{sec:appendix-proofs}
\paragraph{Notation.} Let $\|\cdot\|$ denote the $\ell_2$ norm for a vector. Let $\|\cdot\|_{op}$, $\|\cdot\|_F$ denote the operator norm and the Frobenius norm of a matrix, respectively. For a positive definite matrix $A$, we use $\|x\|_A = \sqrt{x^\top A x}$ to denote the matrix norm induced by $A$, and $\|x\|_{A}^* = \sqrt{x^\top A^{-1}x}$ to denote the dual norm of the induced matrix norm. For a matrix $A$, $A^{-1}$ is the inverse of $A$ if $A$ is full rank; otherwise, $A^{-1}$ is taken to be the Moore-Penrose pseudoinverse. Finally, $\bvec{\cdot}$ denotes the row-major vectorization of a given matrix, and $\otimes$ denotes the Kronecker product between two matrices 
\subsection{Reduction from Non-convex Optimization to Online Convex Optimization}\label{sec:reduction}
In this section, we give more details for the reduction of non-convex optimization to online convex optimization for completeness. We use the framework of \cite{ggt}, though many related results exist in the optimization literature. The algorithm is stated in Algorithm~\ref{alg:reduction}, where we optimize the non-convex function $f$ by iteratively optimizing subproblems $f_t$ that are strongly convex. Given any OCO algorithm $\A$, in each episode, we first construct $f_t$ and then use $\A$ to optimize it for $N$ time steps. 
\begin{algorithm}
\begin{algorithmic}[1]
\caption{Reduction from Non-convex optimization to Online Convex Optimization}
\label{alg:reduction}
\STATE Input: initial $x_1$, time horizon $T$, episode length $N$, smoothness parameter $L$, online convex optimization algorithm $\mathcal{A}$.
\FOR {$t=1,\ldots,T$}
\STATE Construct $f_t(x) = f(x) + L\|x - x_t\|^2$.
\STATE Initialize $\mathcal{A}$ to start at $x_t$, and set $x_t^1= x_t$.
\FOR {$n=1, \ldots, N$}
\STATE Play $x_t^n$, receive stochastic gradient $\tilde{\nabla}f_t(x_t^n)$, construct $g_t^n(x) = \tilde{\nabla}f_t(x_t^n)^\top(x)$.
\STATE Update $x_t^{n+1} = \mathcal{A}(g_t^1, \ldots, g_t^n)$. 
\ENDFOR
\STATE Update $x_{t+1} = \frac{1}{N}\sum_{n=1}^N x_t^n$.
\ENDFOR
\STATE Output iterate $x_{t^*} = \argmin_{t\in [T+1]} \|\nabla f(x_t)\|$.
\end{algorithmic}
\end{algorithm}
We state the convergence guarantees in terms of the adaptive ratio, defined below.
\begin{definition}(Adaptive ratio.) Let $\A$ be an algorithm, and consider a convex function $f$. Given a stochastic gradient oracle with variance bounded by $\sigma^2$, let $x_\mathcal{A}$ be the output of $\A$ with at most $T$ oracle calls, and let $x^*\in \argmin_x f(x)$. Define the adaptive ratio of $\A$ as
$$
\mu_\mathcal{A}(f) = \frac{f(x_\mathcal{A}) - f(x^*)}{\|x_1 - x^*\|\frac{\sigma}{T}}.
$$
\end{definition}
The adaptive ratio captures the performance of $\A$ relative to SGD. For certain algorithms, such as AdaGrad \cite{duchi2011adaptive}, the adaptive ratio can be as small as $\frac{1}{\sqrt{d}}$. For more discussions on this notion see \cite{ggt}.
\begin{theorem}(Theorem A.2 in \cite{ggt}) \label{cor:nonconvex-to-convex}
Consider a non-convex function $f$, and suppose $f$ is $L$-smooth and bounded: $\|\nabla^2 f(x)\|_2 \le L$ and $\max_{x, y} f(x) - f(y) \le M$. Additionally suppose we have access to a stochastic gradient oracle with variance bounded by $\sigma^2$. Let $\mu = \max_t \mu_\mathcal{A}(f_t)$. Then Algorithm \ref{alg:reduction} run with $T = \frac{12ML}{\epsilon^2}$ and $N = \frac{48\mu^2\sigma^2}{\epsilon^2}$ returns a point $x_{t^*}$ such that 
$$
\E[\|\nabla f(x_{t^*})\|] \le \eps.$$
The total number of calls to the stochastic gradient oracle is bounded by $T\cdot N = O(\mu^2\sigma^2/\epsilon^4)$.
\end{theorem}

\subsection{Proof of Lem.~\ref{lem:rho}}\label{sec:rho-proof}
\begin{proof}
Let $H_T\in \R^{T\times d}$ denote the matrix of stacked gradients, where the $t$-th row of $H_T$ is $g_t$. Then $H_T^\top H_T = G_T$, and FD iteratively sketches $H_T$. Let $H_T = U\Sigma V^\top$ be the SVD of $H_T$, and let $H_{T, k} = U_k\Sigma_k V_k^\top$ denote the best rank-$k$ approximation of $H_T$, where $U_k, V_k$ are the first $k$ columns of the matrices, and $\Sigma_k$ is the upper left $k\times k$ submatrix of $\Sigma$. By the proof of Theorem 1.1 in \cite{ghashami2016frequent}, we have
\begin{align*}
    \rho_{1:T} \le \min_{k < \ell} \frac{\|H_T - H_{T, k}\|_F^2}{\ell-k}= \min_{k < \ell} \frac{\sum_{i=k+1}^d \lambda_i(H_T^\top H_T)}{\ell-k}= \min_{k < \ell} \frac{\sum_{i=k+1}^d \lambda_i(G_T)}{\ell-k}\le \sum_{i=\ell}^d \lambda_i(G_T)\,\, ,
\end{align*}
where the last inequality follows by choosing $k=\ell-1$.
\end{proof}

\subsection{Proof of Observation \ref{observation:counterexample}}\label{sec:adafd-cex}
\begin{proof}
Let $\Sigma  = \E\ha{g_tg_t^\top}$ denote the covariance of the gradients, and $\lambda_i = \lambda_i(\Sigma)$ denote its $i$-th eigenvalue. By definition, $g_t$ has the following distribution: $g_t = w_i$ with probability $\lambda_i$. At iteration $t$, we have the current sketch $\bar{G}_{t-1}\in \R^{\ell\times d}$, and we receive the new gradient $g_t$. Ada-FD uses $\bar{G}_{t-1}+\delta I $ as their preconditioner.

We first show that under the distribution of the cost functions, if $\bar{G}_{t-1}$ has rank $\ell-1$, then $\E\ha{\rho_t|\bar{G}_{t-1}} \ge \sum_{i=\ell}^r \lambda_i$. Let $\bar{G}_{t-1} = U\Sigma V^\top$ be the SVD of $\bar{G}_{t-1}$, and $v_i$ be the $i$-th row of $V\in \R^{\ell-1\times d}$. Let $N_{t-1} = W\setminus \{v_1, \ldots, v_{\ell-1}\}$ be the set of basis vectors not in the row space of  $\bar{G}_{t-1}$, then $|N_{t-1}| = r-\ell+1$. If $g_t \in \text{span}(v_1, \ldots, v_\ell)$, then $\rho_t = 0$; otherwise $\rho_t = 1$, with probability $\sum_{i: w_i\in N_{t-1}} \lambda_i \ge \sum_{i=\ell}^r \lambda_i$. 

We proceed to bound the probability that $\bar{G}_{t-1}$ has rank $\ell'\le \ell-2$. Note that this event is equivalent to having fewer than $\ell-1$ distinct vectors drawn from $W$. Let $I_i$ be the indicator variable for drawing $w_i$ in the first $t-1$ rounds, then we can obtain the expected number of distinct vectors as follows
\begin{align*}
    \E\ha{\sum_{i=1}^r I_i} = \sum_{i=1}^r \E\ha{I_i} = \sum_{i=1}^r 1 - (1-\lambda_i)^{t-1}.
\end{align*}
We consider the random variable $r-\sum_{i=1}^r I_i$, and by Markov's inequality,
\begin{align*}
    \Pr\left[r - \sum_{i=1}^r I_i \ge r-\ell+2\right] \le \frac{\sum_{i=1}^r (1-\lambda_i)^{t-1}}{r-\ell+2}\le \frac{r(1-\lambda_r)^{t-1}}{2} \,\, .
\end{align*}
Note that this is exactly the probability of having fewer than $\ell-1$ distinct vectors in the first $t-1$ draws. We conclude that for $t \ge \log r/\lambda_r + 1$, $\Pr\ha{\text{rank}(\bar{G}_{t-1}) = \ell-1 } \ge \frac{1}{2}$. This implies that $\E\ha{\rho_t} \ge \sum_{i=\ell}^r \lambda_i /2$ after an initial log number of rounds, and assuming $T\ge 2\log r/\lambda_r$,
\begin{align*}\E\ha{\sum_{t=1}^T \rho_t} &\ge \E\ha{\sum_{t=T/2+1}^T \rho_t} \ge  \frac{T}{4}\sum_{i=\ell}^r \lambda_i.\,\, \\
\end{align*}
Similarly,
\begin{align*}
    \E\ha{\sum_{t=1}^T \sqrt{\rho_t}} &\ge \E\ha{\sum_{t=T/2+1}^T \sqrt{\rho_t}} \ge  \frac{T}{4}\sum_{i=\ell}^r \lambda_i,
\end{align*}
where the second inequality holds because $\rho_t = 0$ or $1$, so $\sqrt{\rho_t} = \rho_t$ for all $t$.
\ignore{
By using a reverse Markov inequality on the random variable $\sum_{t=1}^T\rho_t \le T$, we obtain
\begin{align*}
    \Pr\ha{\sum_{t=1}^T\rho_t \ge \frac{T}{8}\sum_{i=\ell}^r \lambda_i} &\ge \frac{\frac{T}{8}\sum_{i=\ell}^r \lambda_i}{T - \frac{T}{8}\sum_{i=\ell}^r \lambda_i}\\
    &\ge \frac{\sum_{i=\ell}^r \lambda_i}{8},
\end{align*}
and hence
\begin{align*}
    \E\ha{\sqrt{\sum_{t=1}^T}\rho_t} &\ge \Pr\ha{\sum_{t=1}^T\rho_t \ge \frac{T}{8}\sum_{i=\ell}^r \lambda_i}\sqrt{\frac{T}{8}\sum_{i=\ell}^r \lambda_i}\\
    &\ge \frac{\sqrt{T}}{16\sqrt{2}}\left(\sum_{i=\ell}^r \lambda_i\right)^{3/2}.
\end{align*}
}
The quantities $\sum_{t=1}^T \rho_t$ and $\sum_{t=1}^T \sqrt{\rho_t}$ correspond to $\Delta_T$ and $\sum_{t=1}^T\sqrt{\sigma_t}$ in Theorem 1 of \cite{wan2021efficient}, respectively. Therefore, under our setting, the expectation of the upper bound in Theorem 1 is at least 
\begin{equation}
\label{eq:upper_bound}
 \eta \E\ha{\max\left\{1, \frac{1+\sqrt{\sum_{t=1}^T \rho_t}}{\delta}\right\}\tr(G_T^{1/2})} + \frac{D^2}{2\eta}\E\ha{\sum_{t=1}^T \sqrt{\rho_t}}.
\end{equation}
If we can tune $\delta$, then the $\max$ function evaluates to at least $1$, and 
\begin{align*}
(\ref{eq:upper_bound}) \ge \eta\E\ha{\tr(G_T^{1/2})} + \frac{D^2}{2\eta}\frac{T}{4}\sum_{i=\ell}^r \lambda_i\ge \eta\E\ha{\sqrt{\tr(G_T)}} + \frac{D^2T}{8\eta}\sum_{i=\ell}^r \lambda_i= \eta\sqrt{T} + \frac{D^2T}{8\eta}\sum_{i=\ell}^r \lambda_i,
\end{align*}
where the last equality holds because $\tr(G_T) = \sum_{t=1}^T \|g_t\|_2^2 = T$. Optimizing $\eta$, we conclude that the regret upper bound for Ada-FD is $\Omega(T^{3/4})$ in expectation. 

\ignore{
Now, suppose we cannot tune $\delta$, then 
\begin{align*}
    (\ref{eq:upper_bound}) &\ge \frac{\eta}{\delta}\E\ha{\frac{\sqrt{\sum_{t=1}^T \rho_t}}{\delta}\tr(G_T^{1/2})} + \frac{D^2T}{8\eta}\sum_{i=\ell}^r \lambda_i\\
&\ge \frac{\eta\sqrt{T}}{\delta}\frac{\sqrt{T}}{16\sqrt{2}}\left(\sum_{i=\ell}^r \lambda_i\right)^{3/2} + \frac{D^2T}{8\eta}\sum_{i=\ell}^r \lambda_i\\
&= \frac{\eta T}{\delta16\sqrt{2}}\left(\sum_{i=\ell}^r \lambda_i\right)^{3/2} + \frac{D^2T}{8\eta}\sum_{i=\ell}^r \lambda_i\\
&= \Omega(T).
\end{align*}
}
\end{proof}

\ignore{
\subsection{Generic Theorem}
\begin{theorem}
Let $G_t = \sum_{s=1}^t g_sg_s^\top$. Suppose the sequences of matrices $\{H_t^+\}_{t=1}^T$ and $\{H_t^-\}_{t=1}^T$ satisfy  $H_t^+ \succeq G_t^{1/2}$ and $c_rH_t^+ \succeq H_t^-$ for every $t$. In addition, the sequences are monotonic: $H_t^+ \succeq H_{t-1}^+$ and $H_t^- \succeq H_{t-1}^-$, then using $H_t^+$ as the preconditioner,
the regret of Algorithm satisfies
\begin{align*}
    \text{Regret}_T &\le 
\end{align*}
\end{theorem}
}

\subsection{Proof details for Section \ref{sec:fd-AdaGrad}, \texttt{S-AdaGrad}} 
\subsubsection{Proof of Theorem \ref{thm:add}}\label{sec:proof-fd-ada}

The following observations are made of Algorithm \ref{alg:dfd}:

\begin{observation}\label{observation:onestep}
Denote by $\bar{U}_t\defeq\ha{U_t;U_t^{\bot}}$. If each $M_t$ is of rank $1$, then
\[
\bar{G}_t+\lambda_{\ell}^{(t)} I = \bar{G}_{t-1}+M_t +\lambda_{\ell}^{(t)} N_t\,\,,
\]
where $N_t=U_t^{\bot}\pa{U_t^{\bot}}^\top$. \ignore{By adding $\rho_{1:t-1}I$ to both sides we may replace $\tilde{C}$ with $F$ in the above equation as well.}
\end{observation}
\begin{proof}
By definition of Algorithm \ref{alg:dfd}, $\bar{G}_{t-1}=B_{t-1}B_{t-1}^{\top}$ is of rank at most $\ell-1$. Under the assumption that $M_t$ is of rank $1$, $\bar{G}_{t-1}+M_t$ is of rank at most $\ell$. Therefore, $\lambda_{\ell+1:d}=0$. Then, we have the following:
\begin{align*}
&\bar{G}_t+\lambda_{\ell}^{(t)}I=U_t\diag \left(\lambda_{[1:\ell]}^{(t)}-\lambda_{\ell}^{(t)}\right)U_t^{\top}+\lambda_{\ell}^{(t)}I\\
&=U_t\diag \left(\lambda_{[1:\ell]}^{(t)}-\lambda_{\ell}^{(t)}\right)U_t^{\top}+\lambda_{\ell}^{(t)}(U_tU_t^{\top}+N_t)\\
&=U_t\diag \lambda_{[1:\ell]}^{(t)}U_t^{\top}+ \lambda_{\ell}^{(t)}N_t\\
&=\bar{U}_t \diag \lambda^{(t)}\bar{U}_t^{\top} +\lambda_{\ell}^{(t)}N_t\\
&=\bar{G}_{t-1}+M_t+\lambda_{\ell}^{(t)}N_t\,\, .
\end{align*}
\end{proof}

\begin{lemma}
\label{lem:fd-update}
Let $\lambda_\ell^{(s:t)}$ denote $\sum_{j=s}^t\lambda_\ell^{(j)}$. Let $\tilde{G}_t\defeq \bar{G}_t+\lambda_{\ell}^{(1:t)}I$, $G_t=\sum_{s=1}^t M_s$, where each $M_t$ is of rank $1$. Let the initial $G_0=\bar{G}_0=0$,  then the following relation between $\tilde{G}_t$ and $G_t$ holds for all $t$:
\begin{align*}
\tilde{G}_t=G_t+\sum_{s=1}^t \lambda_{\ell}^{(s)} N_s\,\, .
\end{align*}
\end{lemma}
\begin{proof}
The lemma follows from induction on $t$. Base case $\tilde{G}_0=G_0$ holds by definition. Suppose the above equation holds for $t-1$. Then,
\begin{align*}
\tilde{G}_t=\bar{G}_t+\lambda_{\ell}^{(1:t)}I&=_1\bar{G}_{t-1}+M_t+\lambda_{\ell}^{(t)}N_t+\lambda_{\ell}^{(1:t-1)}I\\
&=\tilde{G}_{t-1}+M_t+\lambda_{\ell}^{(t)}N_t\\
&=_2G_{t-1}+\sum_{s=1}^{t-1}\lambda_{\ell}^{(s)}N_s +M_t+\lambda_{\ell}^{(t)}N_t\\
&=G_t+\sum_{s=1}^t \lambda_{\ell}^{(s)} N_s\,\, ,
\end{align*}
where $=_1$ follows from Observation \ref{observation:onestep} and $=_2$ follows from induction hypothesis.
\end{proof}

\begin{remark}
\label{rmk:fd-inequality}
Note that the above lemma immediately provides an approximate isometry $\bar{G}_t\preceq G_t\preceq \tilde{G}_t$.
\end{remark}

Now, we return to the proof of Thm.~\ref{thm:add}.

\begin{proof}
First, we make the following observation of Algorithm \ref{alg:AdaGrad}:
\begin{observation}
\label{obs:span}
By specification of Algorithm \ref{alg:dfd},\ref{alg:AdaGrad}, $g_t\in \text{span}(\tilde{G}_t)$, $\forall t$. 
\end{observation}
We follow the standard AdaGrad \cite{duchi2011adaptive, hazan2016introduction} analysis. By algorithm specification,
\begin{align*}
y_{t+1}-x^*&=x_t-x^*-\eta \tilde{G}_t^{-1/2}g_t\,\, ,\\
\tilde{G}_t^{1/2}(y_{t+1}-x^*)&=\tilde{G}_t^{1/2}(x_t-x^*)-\eta \tilde{G}_t^{1/2}\tilde{G}_t^{-1/2}g_t=_1\tilde{G}_t^{1/2}(x_t-x^*)-\eta g_t\,\, ,
\end{align*}
where $=_1$ follows from Observation \ref{obs:span}. 

With standard AdaGrad analysis, we can bound regret $\regret_T$ above by the sum of the diameter bound and the gradient bound:
\begin{align*}
\underbrace{\frac{1}{2\eta} \sum_{t=1}^T \|x_t-x_*\|^2_{\tilde{G}_t^{1/2}-\tilde{G}_{t-1}^{1/2}}}_{R_D}+\underbrace{\frac{\eta}{2}\sum_{t=1}^T \|g_t\|_{F_t^{-1/2}}}_{R_G}\,\,.
\end{align*}
Note that by algorithm specification, we have $\forall t$,
\begin{align*}
\tilde{G}_t=\bar{G}_t+\rho_{1:t}I=_1\tilde{G}_{t-1}+g_tg_t^{\top}+\rho_tN_t\succeq \tilde{G}_{t-1}+g_tg_t^{\top},
\end{align*}
where $=_1$ follows from the proof of Lemma \ref{lem:fd-update}. In particular, $\tilde{G}_t\succeq \tilde{G}_{t-1}$. 

Using Remark \ref{rmk:fd-inequality}, the gradient norm term in the regret bound can be further bounded by 
\begin{align*}
R_G=\frac{\eta}{2}\sum_{t=1}^T g_t^{\top}\tilde{G}_t^{-1/2}g_t\le \frac{\eta}{2}\sum_{t=1}^T g_t^{\top}G_t^{-1/2}g_t\le \eta\tr\left(G_T^{1/2}\right)\,\, ,
\end{align*}
where the last inequality follows from Lemma 10 of \citet{duchi2011adaptive}.\footnote{The FTL-BTL lemma \citep{kalai2005efficient} alone is not sufficient to justify this inequality, at least interpreting $G^{-1/2}$ as $\pa{G^{1/2}}^+$. However, \citet{duchi2011adaptive} rely on concavity of $X\mapsto \tr X^{1/2}$ to show a semidefinite version of the statement.} The diameter norm term in the regret bound can be bounded by
\begin{align*}
R_D&=\frac{1}{2\eta} \sum_{t=1}^T \|x_t-x_*\|^2_{\tilde{G}_t^{1/2}-\tilde{G}_{t-1}^{1/2}}\le_1 \frac{D^2}{2\eta}\tr\left(\tilde{G}_T^{1/2}\right)\\
&\le \frac{D^2}{2\eta}\tr\left(\left(G_T+\rho_{1:T}I\right)^{1/2}\right) \ignore{\xc{Is this an inequality?}} \\
&\le_2 \frac{D^2}{2\eta}\left(\tr G_T^{1/2}+\tr(\rho_{1:T}I)^{1/2}\right)\,\, ,
\end{align*}
where $\le_1$ follows from monotonicity of $\tilde{G}_t$'s, $\|\cdot\|_{op}\le \tr(\cdot)$ for positive semidefinite matrices, and linearity of $\tr(\cdot)$, and $\le_2$ follows from that for $X\in\R^d, X\succeq 0$, $\tr(X+\sigma I_d)^{1/2}\le \tr(X^{1/2})+d\sqrt{\sigma}$.
\ignore{\vlad{re: XC, yes. I'd be explicit in comments that $G_T+\rho_{1:T}I\succeq F_T$}}
Combining, we have
\begin{align*}
\regret_T\le \frac{d\sqrt{\rho_{1:T}}+\tr G_T^{1/2}}{2\eta}D^2+\eta\tr\left(G_T^{1/2}\right)= D\left(\sqrt{2}\tr G_T^{1/2}+d\sqrt{\frac{\rho_{1:T}}{2}}\right)\,\, ,
\end{align*}
where the last equality is established by choosing $\eta=\frac{D}{\sqrt{2}}$. 
\end{proof}

\subsubsection{Proof of Corollary \ref{thm:improved}}\label{sec:proof-thm-improved}
\begin{proof}Following the proof of Theorem \ref{thm:add}, we have
\begin{align*}
\regret_T\le \frac{D^2\tr \tilde{G}_t^{1/2}}{2\eta}+\eta \tr G_T^{1/2}\,\, ,
\end{align*}
Denote the accumulated error term
\begin{align*}
E=\sum_{t=1}^T \rho_t N_t\,\, .      
\end{align*}
Then, by Lemma \ref{lem:fd-update} and sub-additivity of $\tr\left(\left(\cdot\right)^{1/2}\right)$ \cite{audenaert2014generalisation},
\begin{align*}
\regret_T\le \left(\frac{D^2}{2\eta} + \eta\right)\tr G_T^{1/2}+\frac{D^2}{2\eta}\tr E^{1/2},
\end{align*}
where it remains to bound the last term. Let $Q$ be a matrix with column vectors $q_i$ that forms an eigenbasis of $E^{1/2}$; this diagonalizes $E$ as well. Notice that
\[
\lambda_i\pa{E}=\sum_{t=1}^T\rho_t q_i^\top N_t q_i\,\,,
\]
and since
\[
\lambda_i^{1/2}\pa{E}=\lambda_i\pa{E^{1/2}}\,\,,
\]
that we can characterize
\[
\tr E^{1/2}=\sum_{i=1}^d \lambda_i^{1/2}(E)=\sum_{i=1}^d\pa{\sum_{t=1}^T\rho_tq_i^\top N_tq_i}^{1/2}\,\,.
\]
Denote $u_{t, i}=q_i^\top N_t q_i$, since $N_t$ is a rank-$(d-\ell)$ projection, $\norm{u_t}_1= d-\ell$. Then $\tr E^{1/2}$ is upper bounded by the value of the program
\begin{align*}
    \max_{u_{t,i}}&\quad\sum_{i=1}^d\pa{\sum_{t=1}^T \rho_t u_{t, i}}^{1/2}\\
    \mathrm{s.t.}&\quad \sum_{i=1}^d u_{t,i}= d-\ell \quad\quad\forall t\in[T]\,\,.
\end{align*}
Note that
\begin{align*}
\sum_{i=1}^d \left(\sum_{t=1}^T \rho_tu_{t,i}\right)^{1/2}\le_1 \sqrt{d} \sqrt{\sum_{i=1}^d \sum_{t=1}^T \rho_tu_{t,i}}=\sqrt{d}\sqrt{\sum_{t=1}^T\rho_t\sum_{i=1}^d u_{t,i}}=_2 \sqrt{d\rho_{1:T}(d-\ell)}\,\, ,
\end{align*}
where $\le_1$ follows from Cauchy-Schwarz, and $=_2$ follows from the  constraint on $\sum_{i=1}^d u_{t,i}$.

Combining, we have
\begin{align*}
\regret_T\le \frac{\sqrt{d(d-\ell)\rho_{1:T}}+\tr G_T^{1/2}}{2\eta} D^2+\eta\tr G_T^{1/2}=D\left(\sqrt{2}\tr G_T^{1/2}+\sqrt{\frac{d(d-\ell)\rho_{1:T}}{2}}\right)\,\, ,
\end{align*}
where the last equality follows from the choice of step size $\eta=\frac{D}{\sqrt{2}}$. 

\ignore{
By symmetry and concavity of the square root this is optimal when all the terms of the outer sum in the objective are equal. 
Since the total mass is $\rho_{1:T}(d-\ell)$, across each of the $d$ outer terms we divide by $d$, square root, and then add $d$ copies. This yields the final objective value of $\sqrt{d(d-\ell)\rho_{1:T}}$.
}
\end{proof}

\subsection{Proof details for Section \ref{sec:fd-shampoo}, \texttt{S-Shampoo}}
\subsubsection{Proof of Theorem \ref{thm:add_shampoo}}\label{sec:proof-shampoo}
\begin{proof} 
First, we establish the following observation and lemma analogous to Observation \ref{observation:onestep} and Lemma \ref{lem:fd-update}:

\begin{observation} [Analogous to Observation \ref{observation:onestep}]
\label{observation:onestep-shampoo}
Let $V_t \Sigma_t^LV_t^\top = \bar{L}_{t-1} + G_tG_t^\top$ be the eigendecomposition of the un-deflated sketch, where $V_t\in \R^{m\times m}$. Suppose $\rank(\Sigma_t^L) = k$, where $k\in [\ell-1, \ell-1+r]$. Write $V_t = [V_t^\parallel \ V_t^\perp]$, where $V_t^\parallel$ contain the first $k$ columns of $V_t$. Then by definition 
\[
\bar{L}_t + \rho_t^L I \succeq \bar{L}_{t-1} + G_tG_t^\top  + \rho_t^L V_t^\perp \pa{V_t^\perp}^\top.
\]
Analogously for the right conditioner, let $W_t \Sigma_t^R W_t^\top = \bar{R}_{t-1} + G_t^\top G_t$, and write $W_t = [W_t^\parallel \ W_t^\perp]$, then 
\[
\bar{R}_t + \rho_t^R I \succeq \bar{R}_{t-1} + G_t^\top G_t  + \rho_t^R W_t^\perp \pa{W_t^\perp}^\top.
\]
\end{observation}

\begin{lemma}\label{lem:fd-update-shampoo}
(Analogous to Lemma \ref{lem:fd-update}) Define $N^L_t = V_t^\perp \pa{V_t^\perp}^\top, N^R_t =W_t^\perp \pa{W_t^\perp}^\top$, then \begin{align*}
\tilde{L}_t \succeq \sum_{s=1}^t G_sG_s^\top + \sum_{s=1}^t \rho_s^L N_s^L + \eps I_m\,\, , \ \ \ \tilde{R}_t \succeq \sum_{s=1}^t G_s^\top G_s + \sum_{s=1}^t \rho_s^R N_s^R + \eps I_n\,\, .
\end{align*}
\end{lemma}

We follow the shampoo proof in \cite{gupta2018shampoo}. Let $x_t = \bvec{X_t}$, $g_t = \bvec{G_t}$, where $\bvec{\cdot}$ denote the row-major vectorization of a given matrix. 

Kronecker product $\otimes$ obeys the following properties as shown in \cite{gupta2018shampoo}:
\begin{lemma}[Lemma 3,4 in \citet{gupta2018shampoo}]
\label{lem:kronecker}
For matrices $A,A',B,B'$ of appropriate dimensions and vectors $u,v$, $L\in\R^{m\times m}, R\in\R^{n\times n}, G\in\R^{m\times n}$, the following properties hold:
\begin{enumerate}
\item $(A\otimes B)(A'\otimes B')=(AA')\otimes(BB')$.
\item $(A\otimes B)^\top = A^\top\otimes B^\top$.
\item $A,B\succeq 0$, $(A\otimes B)^{-1}=A^{-1}\otimes B^{-1}$.
\item $A\succeq A', B\succeq B'$, then $A\otimes B\succeq A'\otimes B'$. 
\item $\trace(A\otimes B)=\trace(A)+\trace(B)$. 
\item $\bvec{uv^\top}=u\otimes v$.
\item $(L\otimes R^\top)\bvec{G}=\bvec{LGR}$.
\end{enumerate}
\end{lemma}

Then the shampoo update is
\[x_{t+1} = x_t - \eta (\tilde{L}_t^{1/4}\otimes \tilde{R}_t^{1/4})^{-1}g_t.
\]

Let $\tilde{H}_t \defeq \tilde{L}_t^{1/4}\otimes \tilde{R}_t^{1/4}$, then by Lemma \ref{lem:kronecker}, $\tilde{H}_t$ is monotone increasing with $t$, since $\tilde{L}_t$ and $\tilde{R}_t$ are monotone by Observation \ref{observation:onestep-shampoo}. Thus, by standard analysis \cite{hazan2016introduction} for Online Mirror Descent (OMD), we can break down the regret into the diameter bound and the gradient bound:
\begin{align*}
\regret_T&\le R_D + R_G,\ \ \ \text{where}
\end{align*}f
\begin{align*}
    R_D = \frac{1}{2\eta} \sum_{t=1}^T \left(\|x_t-x^*\|_{\tilde{H}_t}^2-\|x_{t+1}-x^*\|_{\tilde{H}_t}^2\right), \ \ \ 
    R_G = \frac{\eta}{2}\sum_{t=1}^T \left(\|g_t\|_{\tilde{H}_t}^*\right)^2.
\end{align*}

We proceed to bound $R_D$ and $R_G$ separately. For $R_D$, 
\begin{align*}
R_D&\leq \frac{1}{2\eta}\sum_{t=1}^T \|x_t-x^*\|_{\tilde{H}_t-\tilde{H}_{t-1}}^2+\|x_1-x^*\|_{\tilde{H}_0}^2\\
&\le \frac{1}{2\eta}\sum_{t=1}^T \|\tilde{H}_t-\tilde{H}_{t-1}\|_{op} \|x_t-x^*\|_2^2+\|x_1-x^*\|_{\tilde{H}_0}^2\\
&\le_1 \frac{D^2}{2\eta}\sum_{t=1}^T \tr(\tilde{H}_t-\tilde{H}_{t-1})+\|x_1-x^*\|_{\tilde{H}_0}^2\\
&\le \frac{D^2}{2\eta} \tr(\tilde{H}_T)\,\, ,
\end{align*}
where $\le_1$ holds since $\tilde{H}_t$'s are increasing in $t$, and we have for positive semidefinite matrices $\trace(\cdot)\ge \|\cdot\|_{op}$. 

Now we try to bound $R_G$. First, we have that 
\begin{lemma} [Lemma 8 in \citet{gupta2018shampoo}] 
If $G\in\R^{m\times n}$ with rank at most $r$, and $g=\bvec{G}$, then $\forall \eps\ge 0$, $\forall t$,
\begin{align*}
\ep I_{mn}+\frac{1}{r}\sum_{s=1}^t g_sg_s^{\top}\preceq \left(\eps I_m+\sum_{s=1}^t G_sG_s^{\top}\right)^{1/2}\otimes\left(\eps I_n+\sum_{s=1}^t G_s^{\top}G_s\right)^{1/2}.
\end{align*}
\end{lemma}

Define $M_t^L\in\R^{m\times m}, M_t^R\in\R^{n\times n}$ by 
\begin{align*}
M^L_t \defeq \sum_{s=1}^t G_sG_s^\top + \sum_{s=1}^t \rho_s^L N_s^L+\eps I_m\,\, , \ \ \ M^R_t \defeq \sum_{s=1}^t G_s^\top G_s + \sum_{s=1}^t \rho_s^R N_s^R+\eps I_n\,\, ,
\end{align*}
then by Lemma \ref{lem:fd-update-shampoo}, \[\tilde{L}_t \succeq M_t^L\,\, ,\ \ \tilde{R}_t \succeq M_t^R\,\, .\] Observe that in addition, 
\[
M_t^L \succeq \eps I_m + \sum_{s=1}^t G_sG_s^\top\,\, ,\ \ M_t^R \succeq \eps I_n + \sum_{s=1}^t G_s^\top G_s\,\, .
\]

Again by Lemma \ref{lem:kronecker},
\begin{align*}
I_m \otimes \left(\eps I_n + \sum_{s=1}^t G_s^\top G_s\right)\preceq I_m \otimes M_t^R\,\, ,\ \ \
\pa{\eps I_m + \sum_{s=1}^t G_sG_s^\top} \otimes I_n\preceq M_t^L \otimes I_n\,\, .
\end{align*}

Combining, we have 
\begin{align*}
    \eps I_{mn} + \frac{1}{r}\sum_{s=1}^t g_sg_s^\top \preceq \pa{M_t^L}^{1/2} \otimes \pa{M_t^R}^{1/2}\preceq \tilde{L}_t^{1/2} \otimes \tilde{R}_t^{1/2}\,\, .
\end{align*}

Define $\hat{H}_t\succ 0$ $\forall t\in[T]$ by 

\[\hat{H}_t \defeq\left(r\eps I_{mn} + \sum_{s=1}^t g_sg_s^\top\right)^{1/2}\preceq \sqrt{r}\tilde{H}_t\,\, .\]

The bound on $R_G$ depends on the following lemma:
\begin{lemma}[Lemma 2 in \citet{gupta2018shampoo}]
Consider a sequence of vectors $\{g_t\}_{t=1}^T$. Given a function $\Phi(\cdot)$ over positive semidefinite matrices, 
\begin{align*}
\sum_{t=1}^T \left(\|g_t\|_{H_t}^*\right)^2\le\sum_{t=1}^T\left(\|g_t\|_{H_T}^*\right)^2+\Phi(H_T)-\Phi(H_0)\,\, ,
\end{align*}
where 
\begin{align*}
H_t=\argmin_{H\succ 0} \left\{\left(\sum_{s=1}^t g_sg_s^{\top}\right)\cdot H^{-1}+\Phi(H)\right\}\,\, .
\end{align*}
\end{lemma}

Let $\Phi(H)\defeq \trace(H)+r\eps\trace(H^{-1})$ and since
\begin{align*}
\argmin_{H\succ 0} \left\{\left(\sum_{s=1}^t g_sg_s^{\top}\right)\cdot H^{-1}+\Phi(H)\right\}
=\argmin_{H\succ 0} \left\{\trace{\left(\hat{H}_t^2H^{-1}+H\right)}\right\}=\hat{H}_t\,\, ,
\end{align*}
the above lemma gives
\begin{align*}
\sum_{t=1}^T \left(\|g_t\|_{\hat{H}_t}^*\right)^2\le\sum_{t=1}^T\left(\|g_t\|_{\hat{H}_T}^*\right)^2+\Phi(\hat{H}_T)-\Phi(\hat{H}_0)\le 2\tr(\hat{H}_T)\,\, ,
\end{align*}
which by inequality of $\hat{H}_t$ and $\tilde{H}_t$ established above, gives
\begin{align*}
R_G\defeq\frac{\eta}{2}\sum_{t=1}^T \left(\|g_t\|_{\tilde{H}_t}^*\right)^2\le \frac{\eta\sqrt{r}}{2}\sum_{t=1}^T \left(\|g_t\|_{\hat{H}_t}^*\right)^2\le \eta\sqrt{r}\tr(\hat{H}_T)\le \eta r\tr(\tilde{H}_T)\,\, .
\end{align*}

Combining the bound on $R_D$ and $R_G$, the overall regret is
\begin{align*}
\regret_T \le R_D+R_G\le \pa{\frac{D^2}{2\eta} + \eta r}\tr(\tilde{H}_T) = \sqrt{2r}D\tr(\tilde{H}_T)= \sqrt{2r}D\tr(\tilde{L}_T^{1/4})\tr(\tilde{R}_T^{1/4})\,\, .
\end{align*}
by the choice of $\eta=\frac{D}{\sqrt{2r}}$ and trace multplicative equality in Lemma \ref{lem:kronecker}. Finally, we  have \ignore{\vlad{below is true, but doesn't follow from any of our existing observations, need to upper bound sketch in terms of ggt plus error}}
\begin{align*}
\tr\pa{\tilde{L}_T^{1/4}} 
&\le_1\tr\left(\bar{L}_T^{1/4}\right)+\tr\left(\left(\rho_{1:T}^L I_m\right)^{1/4}\right)\\
&\le_2 \tr\pa{\pa{\sum_{t=1}^TG_tG_t^\top + \eps I}^{1/4}} + m\pa{\rho_{1:T}^L}^{1/4}\\
&=\tr\pa{L_T^{1/4}}+m\pa{\rho_{1:T}^L}^{1/4},
\end{align*}
where $\le_1$ follows from definition of $\tilde{L}_T$ in Algorithm \ref{alg:shampoo} and subadditivity of $\tr\left(\left(\cdot\right)^{1/4}\right)$ \cite{audenaert2014generalisation}, $\le_2$ follows from Remark \ref{rmk:fd-inequality}. Similarly,
\begin{align*}
\tr\pa{\tilde{R}_T^{1/4}}\le \tr\pa{R_T^{1/4}}+n\pa{\rho_{1:T}^R}^{1/4}.
\end{align*}
\end{proof}

\subsubsection{Proof of Lemma \ref{lem:fd-update-shampoo}}
\begin{proof}
We will show the first inequality as the second inequality holds analogously. For $t=0$, $\tilde{L}_0=\eps I_m$ by definition of algorithm. Suppose the first inequality holds for $t$. Consider $t+1$:
\begin{align*}
\tilde{L}_{t+1}&=\bar{L}_{t+1}+\rho_{1:t+1}^L I_m \\
&\succeq_1 \bar{L}_t +G_{t+1}G_{t+1}^\top +\rho_{t+1}^L N_{t+1}+\rho_{1:t}^L I_m\\
&=\tilde{L}_t+G_{t+1}G_{t+1}^\top +\rho_{t+1}^L N_{t+1} \\
&\succeq_2 \sum_{s=1}^{t+1} G_sG_s^\top +\sum_{s=1}^{t+1}\rho_s^LN_s^L+\eps I_m\,\, , 
\end{align*}
where $\succeq_1$ follows from Observation \ref{observation:onestep-shampoo} and $\succeq_2$ follows from induction hypothesis. 
\end{proof}

\section{Training Settings}\label{sec:appendix-training-general}

For repeatable, standard evaluation on modern, competitive tasks we use \texttt{init2winit} \citep{init2winit2021github} for Jax \citep{jax2018github} implementations of architectures in Flax \citep{flax2020github} and standard dataset preprocessing built on top of TFDS \citep{TFDS}. We rely on standard scientific packages for conducting our work \citep{Waskom2021,Hunter:2007,harris2020array,2020SciPy-NMeth}. Our source code will be released after publication. 

Neural net architecture settings are taken from the default settings of the \texttt{init2winit} library at hash \texttt{e337ffe} \citep{init2winit2021github}, which reference the MLCommons specifications provided at \citet{mlcommons},
including the MLPerf ResNet-50 variant \citep{reddi2019mlperf}, Conformer, and GNN. The Distributed Shampoo implementation was run at hash \texttt{83e6e62} in the repository referenced by \citet{anil2020scalable}.

Throughout, weight decay is applied using its decopled variant \citep{loshchilov2017decoupled}.

We requested a Shampoo tuning script from \citet{dayma2022, anil2020scalable}, which fixed several parameters for Shampoo outside the usual defaults. We tuned on a cluster of TPUv4s, with 16 TPUv4s per trial in data-parallel mode.

\begin{itemize}
    \item Block size was already set to 1024. As mentioned in Sec.~\ref{sec:spectral-analysis}, we kept this change for consistency in covariance factor size across architectures.
    \item Preconditioning was set to start 101 steps into training (\texttt{start\_preconditioning\_step}).
    \item Preconditioners were updated every 10 steps instead of every step for speed (\texttt{preconditioning\_compute\_steps} is 10).
    \item The grafting type, which controls the per-tensor learning rate schedule, was set to \texttt{RMSPROP\_NORMALIZED}, which applies RMSProp \citep{hinton2012neural} over unit-normalized gradients.
    \item \texttt{moving\_average\_for\_momentum} was activated (so the final updates are computed as $\beta_1 \mu_t + (1-\beta_1) g_t$, where $\mu_t$ is the momentum term and $g_t$ is the preconditioned update.
    \item The virtual batch size, used to compute batch norm statistics, was set to $32$ (the full per-step minibatch size was $1024$, but this enables data-parallel training).
\end{itemize}

Also from the provided script, we used a linear warmup rampup starting from 0 to the nominal learning rate hyperparameter setting, followed by a cosine decay schedule, with the transition happening 5\% of the way into training (the learning rate monotonically increases, then montonically decreases, as the cosine schedule has a quarter-period set to the number of training steps).

Then, we performed tuning using random hyperparameter search over the space defined in Tbl.~\ref{table:spectral-analysis-hparams}. We ran the batch sizes and number of steps provided in the scripts, which were 256, 512, 1024 for Conformer, GNN, and ResNet-50, respectively, for about 162, 117, 199 epochs, respectively.

Shampoo is automatically configured with grafting parameters, which we search over \citep{agarwal2020disentangling}.

\begin{table}
\caption{The search space for hyperparameters for tuning Shampoo on our NN architectures for the Kronecker-factored covariance optimization. Note that the search space explores one less momentum, not momentum directly. Label smoothing was only applied to ImageNet. We sample uniformly either from linear or logscale among the ranges specified with 100 trials, and select the best one according to validation accuracy.
}
\label{table:spectral-analysis-hparams}
\begin{center}
\begin{tabular}{ccc}
		\hline
		  Hyperparameter    & Range & Log scale?
\\
\hline
Learning rate $\eta$ & $[10^{-4},10^{-2}]$ & \checkmark\\
Weight decay $\gamma$ & $[10^{-2}, 1]$ & \checkmark\\
Momentum $1-\beta_1$ & $[10^{-2}, 10^{-1}]$ & \checkmark\\
Label smoothing & $[0, 0.2]$ &\\
\hline
\end{tabular}
\end{center}
\end{table}

\section{ResNet-50 Settings}\label{sec:resnet-settings}

For training ImageNet, we mostly inherited the settings of Appendix~\ref{sec:appendix-training-general} for Shampoo tuning, but made some minor modifications, namely adding of the second moment decay ($\beta_2$), widening the search space, and, for computational reasons, performing a shortened run of only 66 epochs for tuning trials. The architecture details remain the same. The learning rate schedule was stretched to this interval, so warmup was still 5\% of the duration, and cosine decay ended learning rate at 0 by the end of the 66 epochs of training. The full search space is elaborated on in Tbl.~\ref{table:generalization-hparams}.

\begin{table}
\caption{The search space for hyperparameters for tuning Shampoo on our architectures for ImageNet hyperparameters. The same space was applied to \texttt{S-Shampoo} with a fixed sketch rank $\ell=256$ for all tensors. Note that we search $1-\beta_1,1-\beta_2$, and not the original hyperparameter. We sample uniformly either from linear or logscale among the ranges specified with 256 trials, and select the best one according to validation accuracy. $(*)$ stands for a discrete uniform choice over four different grafting rates, based on AdaGrad, RMSProp, and normalized versions of the two. The gradient clipping norm is similarly discrete.
}
\label{table:generalization-hparams}
\begin{center}
\begin{tabular}{ccc}
		\hline
		  Hyperparameter    & Range & Log?
\\
\hline
Learning rate $\eta$ & $[10^{-4},10^{-2}]$ & \checkmark\\
Weight decay $\gamma$ & $[10^{-3}, 0.1]$ & \checkmark\\
Momentum $1-\beta_1$ & $[10^{-4}, 10^{-1}]$ & \checkmark\\
2\textsuperscript{nd} moment $1-\beta_2$ & $[10^{-4}, 10^{-1}]$ & \checkmark\\
Label smoothing & $[0, 0.2]$ &\\
Dropout Rate & $[0, 0.2]$ &\\
Grafting Type & $(*)$ &\\
Gradient Clip $L_2$ & $\ca{1, 10, 10^2, 10^3}$ &\\
\hline
\end{tabular}
\end{center}
\end{table}

To tune Adam, a first order method, we considered mostly the same nominaly hyperparameters (where $\beta_2$ refers to second moment momentum now), except grafting, which instead we replaced with a search over the warmup duration, summarized in Tbl.~\ref{table:adam-gen-hparams}.

\begin{table}
\caption{The search space for hyperparameters for tuning Adam on our architectures for ImageNet hyperparameters. The same caveats as in Tbl.~\ref{table:generalization-hparams} apply. Also tuned with 256 trials.
}
\label{table:adam-gen-hparams}
\begin{center}
\begin{tabular}{ccc}
		\hline
		  Hyperparameter    & Range & Log?
\\
\hline
Learning rate $\eta$ & $[10^{-4},10^{-2}]$ & \checkmark\\
Weight decay $\gamma$ & $[10^{-3}, 0.1]$ & \checkmark\\
Momentum $1-\beta_1$ & $[10^{-4}, 10^{-1}]$ & \checkmark\\
2\textsuperscript{nd} moment $1-\beta_2$ & $[10^{-4}, 10^{-1}]$ & \checkmark\\
Label smoothing & $[0, 0.2]$ &\\
Dropout Rate & $[0, 0.2]$ &\\
Warmup Duration & $[2\%, 10\%]$ of training &\\
Gradient Clip $L_2$ & $\ca{1, 10, 10^2, 10^3}$ &\\
\hline
\end{tabular}
\end{center}
\end{table}

Full evaluation of the selected best hyperparameters for each setting was performed with the classical 90-epoch setting, with the learning rate schedule correspondingly stretched.

We provide the full training curves in Fig.~\ref{fig:imagenet-long}.

\begin{figure}
\centering
\includegraphics[width=\columnwidth]{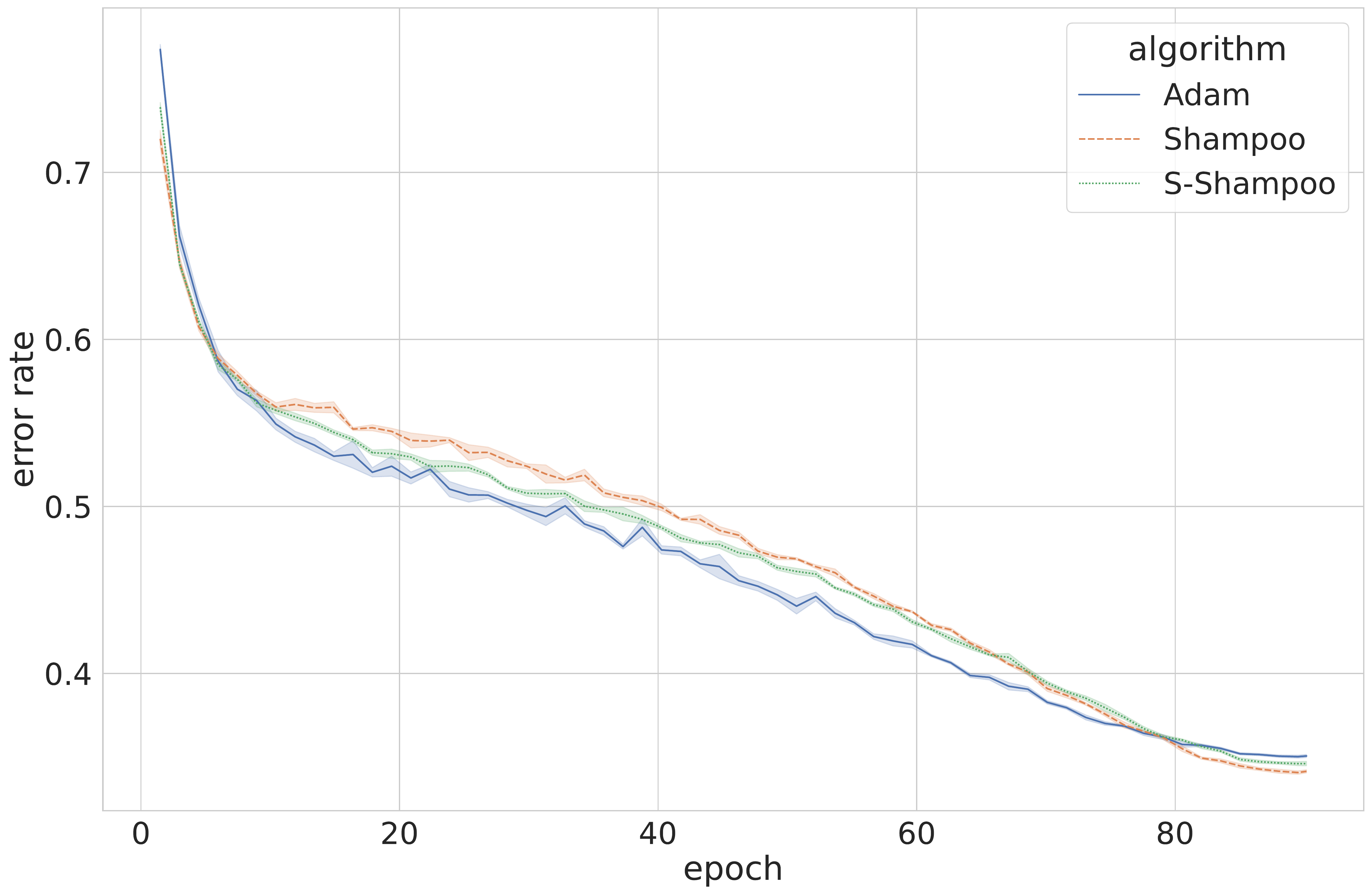}
\caption{Full test set curves for imagenet.}\label{fig:imagenet-long}
\end{figure}

\section{Conformer Settings}\label{sec:libri-hparams}

The Conformer architecture was used from the MLCommons specification as described in Appendix~\ref{sec:appendix-training-general}, with the following fixed additional settings: 1024 batch size, 100 epochs of training, 5\% of training used for linear warmup with cosine decay of learning rate. We fixed gradient clipping at a value of $10$, without which we noticed Adam curves were very volatile. We set the \texttt{eigh} parameter for Shampoo to true (we found that it did not make a difference in a few sample runs' loss curves, as an alternative to the iterative $p$-th inverse root routine in Shampoo, but used it instead since we believe it has better numerical stability). The hyperparameters we searched over for all optimizers are described in Tbl.~\ref{tbl:conformer-hparams}.

\begin{table}
\caption{The search space for hyperparameters for tuning Shampoo, Adam, and \texttt{S-Shampoo} for Conformer with fixed $256$ rank. Here we fixed the grafting type to be RMSProp. During initial runs of the baselines, we noticed that Adam preferred larger learning rates, so we changed and reran its space for $\eta$ to be $10\times$ that of Shampoo, namely $[10^{-4},10^{-2}$, still searching over logspace. We also stopped any hyperparameter trial which did not go below 0.875 WER after 5000 training steps.
}\label{tbl:conformer-hparams}
\begin{center}
\begin{tabular}{ccc}
		\hline
		  Hyperparameter    & Range & Log?
\\
\hline
Learning rate $\eta$ & $[10^{-5},10^{-3}]$ & \checkmark\\
Momentum $1-\beta_1$ & $[10^{-3}, 10^{-1}]$ & \checkmark\\
2\textsuperscript{nd} moment $1-\beta_2$ & $[10^{-3}, 10^{-1}]$ & \checkmark\\
Weight decay $\gamma$ & $[10^{-4}, 10^{-2}]$ & \checkmark\\
Dropout Rate & $[0, 0.2]$ & \checkmark\\
\hline
\end{tabular}
\end{center}
\end{table}

\section{GNN Settings}\label{sec:gnn-hparams-settings}

The GNN architecture was used from the MLCommons specification as described in Appendix~\ref{sec:appendix-training-general}, with the following fixed additional settings: 1024 batch size, 30 epochs of training, 5\% of training used for linear warmup with cosine decay of learning rate, 0.05 dropout, and we set the \texttt{eigh} parameter for Shampoo to true as in Appendix~\ref{sec:libri-hparams}.

Then we searched a hyperparameter space for Shampoo, Adam, and \texttt{S-Shampoo} as described in Tbl.~\ref{tbl:gnn-hparams}.

\begin{table}
\caption{The search space for hyperparameters for tuning Shampoo and \texttt{S-Shampoo} for GNN with fixed regularization settings (due to resource constraints, we only ran 128 samples from the grid here). Here we fixed the grafting type to be RMSProp and did not use gradient clipping, unlike Tbl.~\ref{table:generalization-hparams}, based on a few trial runs of the Shampoo baseline from which we determined we could reduce the hyperparameter space.
}\label{tbl:gnn-hparams}
\begin{center}
\begin{tabular}{ccc}
		\hline
		  Hyperparameter    & Range & Log?
\\
\hline
Learning rate $\eta$ & $[10^{-4},10^{-2}]$ & \checkmark\\
Momentum $1-\beta_1$ & $[10^{-3}, 0.5]$ & \checkmark\\
2\textsuperscript{nd} moment $1-\beta_2$ & $[10^{-3}, 0.5]$ & \checkmark\\
Weight decay $\gamma$ & $[10^{-3}, 0.5]$ & \checkmark\\
\hline
\end{tabular}
\end{center}
\end{table}

\section{Step-skipping}\label{sec:step-skipping}
In this section, we provide some theoretical justification for step-skipping. We first derive the regret bound of AdaGrad with step skipping, named Generic Epoch AdaGrad. The additional regret incurred is expressed as an error term. Then, we describe a setting where the error term admits a simple bound, showing that step-skipping incurs at most an extra $\log T$ time dependence on the regret.  
\subsection{Adversarial losses}

Consider a generalized epoching AdaGrad with $K$ fixed update points $t_k$, such that $t_1=0$ and $t_K=T$.

\begin{algorithm}[h!]
\caption{Generic Epoch AdaGrad}
\label{alg:epoch-AdaGrad}
\begin{algorithmic}[1]
\STATE Input: $\eta, T, \ca{t_k}_{k=1}^K, G_0\succ 0$, convex closed set $\K$.
\STATE Initialize: $x_1$.
\FOR{$k=1,\dots,K-1$}
\FOR{$t=t_{k}+1,\cdots,t_{k+1}$}
\STATE Play $x_t$, receive $f_t$ loss with gradient $g_t$.
\STATE Update $G_{t}= G_{t-1} + g_tg_t^\top$.
\STATE Update $x_{t+1} =\Pi_{\K} [x_t - \eta G_{t_k}^{-1/2}g_t]$.
\ENDFOR
\ENDFOR
\end{algorithmic}
\end{algorithm}

\begin{theorem}\label{thm:generic-epoch-AdaGrad}
Generic Epoch AdaGrad (Alg.~\ref{alg:epoch-AdaGrad}) with fixed update points $\ca{t_k}_{k=1}^K$ satisfies
\begin{align*}
R_T\le \frac{D^2}{\eta}\tr G_T^{1/2}+\frac{\eta}{2}\pa{2\tr G_T^{1/2}-2\tr G_0^{1/2}+\sum_{k=1}^K\epsilon_k}\,\,,
\end{align*}
where the error terms $\epsilon_k$ are given by
\begin{align*}
\epsilon_k&=\tr\pa{ G_{t_k}^{-1/2}S_k G_{t_k}^{-1/2} A_k}\,\,,\\
S_k&=\int_0^\infty \exp\pa{-\tau G_{t_k}^{1/2}}A_k\exp\pa{-\tau G_{t_k}^{1/2}} \d{\tau}\,\,,\\
A_k&=G_{t_{k+1}}-G_{t_k}\,\,. 
\end{align*}
\end{theorem}

\begin{proof}[Proof of Theorem \ref{thm:generic-epoch-AdaGrad}]
First we start with the usual decomposition.

\begin{lemma}\label{lem:basic-decomposition}
Consider arbitrary adversarial convex losses $f_t$. Without projection, the regret $R_T$ relative to a comparator $x_*$ with $D=\max_t\norm{x_t-x_*}_2$, for generic epoch AdaGrad with fixed update points $t_k$ is given by
\[
R_T\le \frac{D^2}{ \eta}\tr G_T^{1/2}+\frac{\eta}{2}\sum_{k=1}^K\sum_{t=t_k+1}^{t_{k+1}}g_t^\top G_{t_k}^{-1/2}g_t\,\,.
\]
\end{lemma}
\begin{proof} [Proof of Lemma \ref{lem:basic-decomposition}]
This follows from the usual AdaGrad analysis since our preconditioners are monotone $G_{t_k}\succeq G_{t_{k+1}}$.
\end{proof}

So we must turn our attention to the gradient bound. We start by noting the following lemmas established in matrix analysis.

\begin{lemma}[Corollary 4.1 in \citep{ando1979concavity}]
The map $f(X)=X^{-1/2}$ is matrix convex over the positive definite domain; i.e., for any two matrices $A, B\succ 0$ and any $\theta\in[0,1]$, we have
\[
\theta f(A)+(1-\theta) f(B)\succeq f(\theta A + (1-\theta)B)\,\,.
\]
\end{lemma}

\begin{lemma}[Theorem V.3.3, Exercise V.3.15 in \citep{bhatia2005}]
\label{claim:matcvx}
Suppose a matrix convex function $F(X)$ is induced by applying $f$ pointwise to its spectrum $F(X)=U\diag [f(\Lambda_{ii})]U^{\dagger}$ with $f\in\mathcal{C}^1(I)$ for some $I\subset\R_+$. Then
\[
F(X)+\partial F(X)(\Delta)\preceq F(X+\Delta)\,\,,
\]
if and only if $F(X)$ is matrix convex, and the linear transformation $\partial F(X)$ is the derivative of $F$ at $X$. 
\end{lemma}

Matrix derivative computation \citep{petersen2008matrix, brockett2015finite} shows that if $F(X)=X^{-1/2}$ then
\[
\partial F(X)(\Delta)= - X^{-1/2}\left[(X^{1/2}\oplus X^{1/2})^{-1}\Delta\right]X^{-1/2}\,\,,
\]
where $(X^{1/2}\oplus X^{1/2})^{-1}\Delta$ is the solution $S$ to the continuous Lyapunov equation $\sqrt{X}S+S\sqrt{X}=\Delta$ as $X\oplus X = I\otimes X + X\otimes I$. For $X\succ 0$, it is known from generic results about Sylvester's equation that the solution $S$ is unique. Since $-X$ is asymptotically stable in the Lyapunov sense, 
\[
S(X, \Delta)=\int_0^{\infty}\exp(-\tau \sqrt{X})\Delta\exp(-\tau \sqrt{X})\d{\tau}\,\,.
\]

With these results from matrix analysis and linear systems, we are ready to bound the gradient term in Lemma~\ref{lem:basic-decomposition}. Consider a single term from the gradient bound in Lemma~\ref{lem:basic-decomposition}, $\sum_{t=t_k+1}^{t_{k+1}}g_t^\top G_{t_k}^{-1/2}g_t$ for fixed $k$.

With $X=G_{t_k}$, $\Delta =A_k= G_{t_{k+1}}-G_{t_k}$, and $f(X)=X^{-1/2}$ consider applying Lemma~\ref{claim:matcvx}. $F(X)\preceq F(X+A_k)-\partial F(X)(A_k)$, so overall
\begin{align*}
&\sum_{t=t_k+1}^{t_{k+1}}g_t^\top G_{t_k}^{-1/2}g_t \\
&\le \sum_{t=t_k+1}^{t_{k+1}}g_t^\top \left[G_{t_{k+1}}^{-1/2}-\partial F(G_{t_k})(A_k)\right]g_t\\
&=\sum_{t=t_k+1}^{t_{k+1}}g_t^\top G_{t_{k+1}}^{-1/2}g_t - \tr\pa{\partial F(G_{t_k})(A_k)\sum_{t=t_k+1}^{t_{k+1}}g_tg_t^\top}\\
&=\sum_{t=t_k+1}^{t_{k+1}}g_t^\top G_{t_{k+1}}^{-1/2}g_t - \tr\pa{\partial F(G_{t_k})(A_k)A_k}\\
&=\sum_{t=t_k+1}^{t_{k+1}}g_t^\top G_{t_{k+1}}^{-1/2}g_t + \tr\left(G_{t_k}^{-1/2}S_kG_{t_k}^{-1/2}A_k\right)\\
&=\sum_{t=t_k+1}^{t_{k+1}}g_t^\top G_{t_{k+1}}^{-1/2}g_t + \epsilon_k\,\,.
\end{align*}

\begin{lemma}[FTL-BTL with errors] 
\label{lem:ftlbtl-error}
Consider arbitrary $\phi_k$ for $k\in 0,1,\cdots,K$. Let $x_k\in\argmin\sum_{j=0}^k\phi_j$ and suppose $\phi_k(x_{k-1})\le \phi_k(x_k)+\delta_k$. Then $\forall K$,
\[
\sum_{k=0}^K\phi_k(x_{k-1})\le \sum_{k=0}^K\phi_k(x_K)+\delta_k\,\,,
\]
with $\delta_0=0$ and $x_{-1}=x_0$.
\end{lemma}
Assume Lemma \ref{lem:ftlbtl-error} and take $\phi_k(X)=\inner{ A_k, X}$ for $X\succeq 0$ and $\phi_0(X)=\inner{G_0, X} +\tr X^{-1}$. Note that 
\begin{align*}
\sum_{j=0}^k\phi_j(X)&=\tr X^{-1}+\sum_{j=0}^k\inner{A_j, X}\,\,.
\end{align*}
In particular, $G_{t_{k+1}}^{-1/2}=\argmin_{X\succeq 0} \sum_{j=0}^k \phi_j(X)$.

Furthermore, with $\delta_k=\epsilon_k$, the condition $\phi_k(G_{t_{k}})\le\phi_k(G_{t_{k+1}})+\delta_k$ is satisfied. 
Lemma~\ref{lem:ftlbtl-error} implies 
\begin{align*}
&\sum_{k=1}^K\sum_{t=t_k+1}^{t_{k+1}}g_t^\top G_{t_k}^{-1/2}g_t=\sum_{k=1}^{K-1}\phi_k\left(G_{t_k}^{-1/2}\right)\\
&= -\phi_0(G_0^{-1/2})+\sum_{k=0}^{K-1}\phi_k\left(G_{t_k}^{-1/2}\right)\\
&\le -\phi_0(G_0^{-1/2}) + \sum_{k=0}^{K-1}\phi_k\left(G_{T}^{-1/2}\right)+\epsilon_k\,\,,
\end{align*}
where $G_{t_0}\defeq G_0$. Lastly, since 
\begin{align*}
\sum_{k=0}^{K-1}\phi_k\left(G_T^{-1/2}\right)=\tr G_T^{1/2}+\tr \pa{G_T^{-1/2}\sum_{k=0}^{K-1}A_k}=2\tr G_T^{1/2},
\end{align*}
we conclude with the desired bound for $R_T$.
\end{proof}

\begin{proof}[Proof of Lemma~\ref{lem:ftlbtl-error}]
By induction on $K$. For $K=0$, $\phi_0(x_{-1})=\phi_0(x_0)$, holding by definition. Suppose that the hypothesis now holds for $K$; it then holds for $K+1$.
\begin{align*}
    \sum_{k=0}^{K+1}\phi_k(x_{K+1})&=\sum_{k=0}^{K}\phi_k(x_{K+1})+\phi_{K+1}(x_{K+1})\\
    &\ge \phi_{K+1}(x_{K+1})+ \sum_{k=0}^{K}\phi_k(x_{K})\\
    &\ge \phi_{K+1}(x_{K}) - \epsilon_{K+1} + \sum_{k=0}^{K}\phi_k(x_{K})\\
    &\ge \phi_{K+1}(x_{K}) - \epsilon_{K+1} + \sum_{k=0}^{K}\phi_k(x_{k-1})-\epsilon_k\\
    &\ge \sum_{k=0}^{K+1}\phi_k(x_{k-1})-\epsilon_k\,\,.
\end{align*}

\end{proof}

\subsection{Simplifying the error}

We want to simplify the term $\epsilon_k$, which is given by
\begin{align*}
\epsilon_k&=\tr\pa{ G_{t_k}^{-1/2}S_k G_{t_k}^{-1/2} A_k}\,\,,\\
S_k&=\int_0^\infty \exp\pa{-\tau G_{t_k}^{1/2}}A_k\exp\pa{-\tau G_{t_k}^{1/2}} \d{\tau}\,\,,\\
A_k&=G_{t_{k+1}}-G_{t_k}\,\,. 
\end{align*}
Next, notice that $X$ and $\exp\pa{-\alpha X^{-1}}$ commute. Then along with linearity of trace, we can established that
\begin{align*}
\epsilon_k=\int_0^\infty \tr\ha{\pa{ \exp\pa{-\tau G_{t_k}^{1/2}}G_{t_k}^{-1/2}A_k}^2}\d{\tau} \,\,.
\end{align*}

\subsection{Towards simpler error}

$\epsilon_k$ can be further simplified and bounded under additional assumptions. Namely,

\begin{assumption}
\label{assum1}
Suppose that w.p. at least $1-\delta/2K$ for some fixed, universal $\beta>0$, we have the inequality $A_k\preceq \beta G_{t_k}$ where $A_k\defeq G_{t_{k+1}}-G_{t_k}$.
\end{assumption}

\begin{assumption}\label{assum:eigs}
Suppose that w.p. at least $1-\delta/2K$, $G_{t_k}$'s are $(\sigma_{\min},\sigma_{\max})$-well-conditioned, i.e. 
\[\lambda_{d}(G_{t_k})\ge \sigma_{\min} t_k \ \ \ \text{and} \ \ \ \lambda_{1}(G_{t_k})\le \sigma_{\max} t_k.\]
\end{assumption}

\begin{remark}
As an example, consider the stochastic linear setting where at each iteration we receive a loss function $\langle g_t,x\rangle$, and $g_t$'s are independent, though not necessarily identically distributed, and satisfies that $2\sigma_{\min}I\preceq \E[g_tg_t^{\top}]\preceq \frac{\sigma_{\max}}{2}I$ and $\|g_t\|_2\le \sqrt{\frac{\sigma_{\max}}{2}}$ almost surely. Then, for $T$ sufficiently large and $t_{k+1}-t_k=O(\log T)$, by matrix Chernoff bounds Assumption~\ref{assum1} and \ref{assum:eigs} are satisfied. 
\end{remark}

First, $\forall X,Y\succeq 0$, the following inequality hold: 

\begin{lemma}\label{lem:tr2}
If $X\preceq Y$ and $A\succeq 0$, then $\tr\ha{\pa{AX}^2}\le \tr\ha{\pa{AY}^2}$.
\end{lemma}

With Lemma~\ref{lem:tr2}, we can bound $\epsilon_k$. With probability at least $1-\delta/2K$,

\begin{align*}
\epsilon_k
&=\int_0^\infty \tr\ha{\pa{\exp\pa{-\tau G_{t_k}^{1/2}}G_{t_k}^{-1/2}A_k}^2}\d{\tau}\\
&\le\beta^2 \int_0^\infty \tr\ha{\pa{\exp\pa{-\tau G_{t_k}^{1/2}}G_{t_k}^{-1/2}G_{t_k}}^2}\d{\tau}\\
&=\beta^2 \int_0^\infty \tr\ha{\pa{\exp\pa{-\tau G_{t_k}^{1/2}}G_{t_k}^{1/2}}^2}\d{\tau}\\
&=\beta^2 \int_0^\infty \tr\pa{\exp\pa{-2\tau G_{t_k}^{1/2}}G_{t_k}}\d{\tau}\,\,,
\end{align*}
where the last step only holds since $X$ and $\exp\left(-\alpha X^{-{1/2}}\right)$ commute.

Next, let $\lambda_i$ denote the $i$-th largest eigenvalue and $\lambda_{-i}$ be the $i$-th smallest. Notice since $\exp\pa{-\tau G_{t_k}^{1/2}}$, $G_{t_k}^{1/2}$, and $G_{t_k}$ are simultaneously diagonalizable, and montonic matrix functions preserve eigenvalue ordering, we have
\begin{align*}
    \lambda_i\pa{\exp\pa{-2\tau G_{t_k}^{1/2}}}&=\exp\pa{-2\tau \lambda_{-i}\left(\pa{G_{t_k}}^{1/2}\right)}\,\,,\\
    \lambda_i\pa{\exp\pa{-2\tau G_{t_k}^{1/2}}G_{t_k}}&=\lambda_i\pa{\exp\pa{-2\tau G_{t_k}^{1/2}}}\lambda_i\pa{G_{t_k}}\,\,.
\end{align*}

Returning to our $\epsilon_k$ bound, rewriting the trace with eigenvalues, we have w.p. at least $1-\delta/2K$,
\begin{align*}
    \epsilon_k&\le_1 \beta^2 \int_0^\infty \sum_i\lambda_i\pa{\exp\pa{-2\tau G_{t_k}^{1/2}}G_{t_k}}\d{\tau}\\
    &=\beta^2\sum_i\lambda_i(G_{t_k})\int_0^\infty \exp\pa{-2\tau \lambda_{-i}(G_{t_k})^{1/2}}\d{\tau}\\
    &=\beta^2\sum_i\frac{\lambda_i(G_{t_k})}{2\lambda_{-i}(G_{t_k})^{1/2}}\,\,,
\end{align*}
where $\le_1$ follows from Tonelli's Theorem.
At this point, we apply Assumption~\ref{assum:eigs} and get that w.p. at least $1-\delta/K$,
\begin{align*}
  \epsilon_k&\le\frac{\beta^2}{\sqrt{t_k\sigma_{\min}}}\sum_{i}\lambda_i(G_{t_k})\\
&\leq\frac{\beta^2}{\sqrt{t_k\sigma_{\min}}}\sum_{i}(t_k\sigma_{\max})^{1/2}\lambda_i(G_{t_k})^{1/2}\\
&=\beta^2\sqrt{\frac{\sigma_{\max}}{\sigma_{\min}}} \sum_{i}\lambda_{i}(G_{t_k})^{1/2}\\
&=\beta^2\sqrt{\frac{\sigma_{\max}}{\sigma_{\min}}}\tr G_{t_k}^{1/2}\,\,.
\end{align*}

Across all epochs, we then have w.p. at least $1-\delta$,
\[\frac{1}{\beta^2}\sqrt{\frac{\sigma_{\min}}{\sigma_{\max}}}\sum_{k=1}^K\epsilon_k\le \sum_{k=1}^K\tr G_{t_k}^{1/2}\le \log T \tr G_T^{1/2}\,\,.\]
Altogether, since $\beta$ is a universal constant, w.p. at least $1-\delta$,
\[
R_T\lesssim \frac{D^2}{\eta}\tr G_T^{1/2}+\eta \sqrt{\frac{\sigma_{\max}}{\sigma_{\min}}}\log T\tr G_T^{1/2}\,\,.
\]
We conclude that in this case, the time dependency of Epoch AdaGrad's regret is only $\log T$ factor worse than that of the original AdaGrad regret.
\subsubsection{Proof of Lemma~\ref{lem:tr2}}
First, for $0\preceq X\preceq Y$, $BXB\preceq BYB$, $\forall B$, since $(Bx)^\top (Y-X)(Bx)\ge 0$, $\forall x$. By cyclic property of trace and taking $B=A^{1/2}X$,
\begin{align*}
    \tr\pa{AXAX}=\tr\pa{X^{1/2}AXAX^{1/2}}\le \tr\pa{X^{1/2}AYAX^{1/2}}\,\,.
\end{align*}
 Continuing,
\begin{align*}
    \tr\ha{\pa{AX}^2}&\le \tr\pa{X^{1/2}AYAX^{1/2}}\\
    &= \tr\pa{Y^{1/2}AXAY^{1/2}}\\
    &\le \tr\pa{Y^{1/2}AYAY^{1/2}}\\
    &=\tr\ha{\pa{AY}^2}\,\,.
\end{align*}

\end{document}